\documentclass{article} 
\usepackage{iclr2027_conference,times}

\usepackage{amsmath,amsfonts,bm}

\def\eqref#1{equation~\ref{#1}}

\def\1{\bm{1}}

\DeclareMathAlphabet{\mathsfit}{\encodingdefault}{\sfdefault}{m}{sl}
\SetMathAlphabet{\mathsfit}{bold}{\encodingdefault}{\sfdefault}{bx}{n}

\newcommand{\E}{\mathbb{E}}

\newcommand{\R}{\mathbb{R}}

\usepackage{hyperref}
\usepackage{url}
\makeatletter
\g@addto@macro\UrlBreaks{\do\0\do\1\do\2\do\3\do\4\do\5\do\6\do\7\do\8\do\9%
  \do\a\do\b\do\c\do\d\do\e\do\f\do\g\do\h\do\i\do\j\do\k\do\l\do\m%
  \do\n\do\o\do\p\do\q\do\r\do\s\do\t\do\u\do\v\do\w\do\x\do\y\do\z%
  \do\A\do\B\do\C\do\D\do\E\do\F\do\G\do\H\do\I\do\J\do\K\do\L\do\M%
  \do\N\do\O\do\P\do\Q\do\R\do\S\do\T\do\U\do\V\do\W\do\X\do\Y\do\Z}
\makeatother

\usepackage{amssymb}
\usepackage{amsthm}
\usepackage{booktabs}
\usepackage[table]{xcolor}
\definecolor{oursbg}{RGB}{223,235,252}
\definecolor{digitbg}{RGB}{241,241,241}
\definecolor{promptbg}{RGB}{248,249,250}
\definecolor{promptedge}{RGB}{214,218,222}
\usepackage{multirow}
\usepackage{graphicx}
\usepackage{capt-of}
\newcommand{\sd}[1]{{\tiny$\pm$#1}}
\newcommand{\hd}{Acc\,$\uparrow$ & ECE\,$\downarrow$ & AUROC\,$\uparrow$ & Brier\,$\downarrow$}
\usepackage{lastpage}
\usepackage{enumitem}

\theoremstyle{definition}
\newtheorem{definition}{Definition}

\theoremstyle{plain}
\newtheorem{proposition}{Proposition}

\newcommand{\emission}{\mathcal{E}}
\newcommand{\trainpath}{\mathcal{T}}

\newcommand{\markmainend}{\phantomsection\label{paper:mainend}}

\title{Beyond Verbalized Confidence: Calibrating Reasoners with Differentiable Readouts}

\author{Chenxiao Fan$^{1,2}$, Chongming Gao$^{1}$, Gangyi Zhang$^{2}$, Leyang Shen$^{3}$, Yaxin Gong$^{1}$,
\\
\textbf{Jiamin Wang$^{1}$, Jiakai Wang$^{2}$, Dong Wang$^{2}$, Yang Liu$^{2}$, Fuli Feng$^{1}$, Xiangnan He$^{1}$}\\
$^{1}$ University of Science and Technology of China \\ 
$^{2}$ Qwen Business Unit of Alibaba \quad $^{3}$ National University of Singapore \\ 
\texttt{simonfan@mail.ustc.edu.cn} \\
}

\iclrfinalcopy 

\begin{document}

\maketitle

\begin{abstract}
Reinforcement learning with verifiable rewards (RLVR) trains reasoning models
to produce correct answers, but does not ensure that their stated confidence
is calibrated. The resulting models are systematically overconfident.
Recent methods train calibration inside the RLVR loop by having the model
state a numerical confidence alongside its answer, but they all obtain the
confidence by sampling it as text.
This choice imposes two costs: a sampled confidence introduces variance and in
practice collapses to a handful of distinct values, and sampling makes the
confidence non-differentiable, forcing the calibration loss through a scalar
reward.
We propose \textbf{CREDO} (Confidence REaDOut) to replace sampling with a
deterministic readout. While RLVR optimizes correctness, CREDO reads the
confidence from a dedicated token pair in the model's output distribution and
trains it by differentiable regression.
CREDO further turns the trained confidence into a signal for accuracy,
weighting rollouts by how far confidence and outcome disagree, so that
accuracy and calibration improve together.
Across mathematical and code reasoning, CREDO attains the best accuracy and
calibration, and the gains extend to abstention and selective prediction.
\end{abstract}

\section{Introduction}
\label{sec:intro}

Reinforcement learning with verifiable rewards (RLVR) has become the standard
recipe for training reasoning models \citep{grpo,deepseekr1}.
Answers in mathematics and code can be checked automatically, so the training
reward is simply whether the answer is correct.
This produces strong reasoners, but the reward carries no signal about how
confident the model should be in any particular answer. Models trained this way
are systematically overconfident
\citep{yang2024verbalized,bereket2025,mei2025}.

Overconfidence has practical costs: downstream systems route answers by their
stated confidence, so a confidently wrong answer passes through the filters
meant to catch it \citep{geifman2017,abstentionbench}.
The remedy is calibration: a model's stated confidence should match how
often it is actually correct, so that when it reports 0.9, roughly nine in ten
of those answers are right \citep{degroot1983,guo2017}.
Calibration does not emerge from a correctness reward, so it has to be trained
for.

For calibration to be trainable, the model must produce a confidence estimate
alongside each answer (Figure~\ref{fig:channels}a).
One approach reads confidence from the model's generation probabilities
\citep{fomicheva2020}, but these reflect the likelihood
of the generated text rather than the probability of being correct.
Recent methods instead ask the model to state a numerical confidence after its
answer, scored by a reward based on a proper scoring rule
\citep{gneiting2007}.
Several methods along this route have refined the reward design and the balance
between correctness and calibration
\citep{rlcr,dcpo,sayself,rewardingdoubt,leng2025}.
Despite their differences, these methods share one choice: the confidence
is a \emph{numeral sampled as text}, which forces the calibration loss to work
through a scalar reward (Figure~\ref{fig:channels}b).

\begin{figure}[t]
\centering
\includegraphics[width=\textwidth]{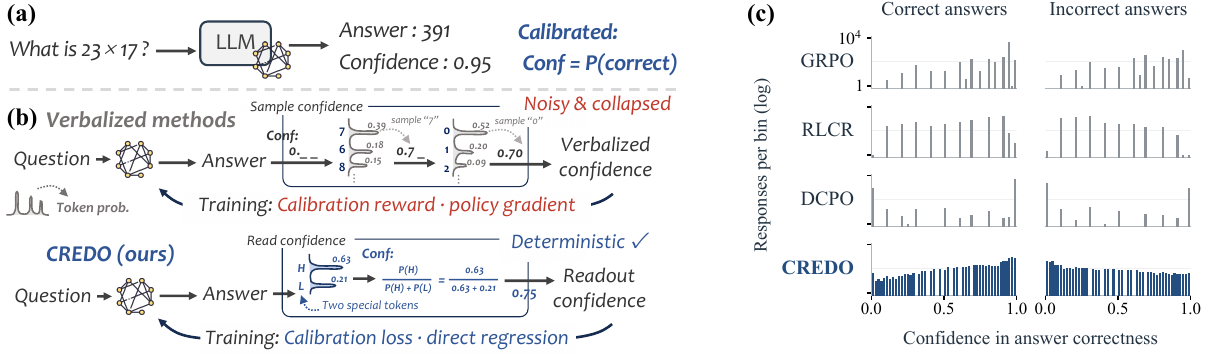}
\caption{\textbf{(a)} The calibration task: a model answers and reports a
confidence; calibration means the reported probability matches observed
correctness. \textbf{(b)} Verbalized methods sample the confidence as text and
train it through a scalar reward. CREDO reads the relative probability of a
pair of reserved tokens and trains it by direct regression from the calibration
loss. \textbf{(c)} Confidence distributions on code. Verbalized reports (GRPO, RLCR, DCPO)
concentrate on a few values shared by correct and incorrect answers, even when
calibration is trained explicitly (RLCR, DCPO). The readout (CREDO) spreads
across the full range.}
\label{fig:channels}
\end{figure}

This choice imposes two costs. First, a sampled confidence introduces variance
that persists as long as the report is stochastic. In practice, the reported
confidence collapses to a handful of distinct values, and correct and incorrect
answers end up sharing the same confidence (Figure~\ref{fig:channels}c).
Second, sampling makes the confidence non-differentiable, forcing the
calibration loss to work through a scalar reward rather than a direct gradient.
The resulting gradient estimate carries additional variance from the report,
making calibration harder to learn.
Both costs follow from verbalizing the confidence; redesigning the reward
cannot remove them.

We propose CREDO (Confidence REaDOut),\footnote{The name also reads as the Latin
\emph{credo}, ``I believe.''} which reads the confidence directly from the
output distribution rather than sampling it.
While RLVR optimizes correctness, CREDO sets aside a pair of reserved tokens
and takes their relative probability as the confidence, trained by
differentiable regression (Figure~\ref{fig:channels}b).
The readout is deterministic and directly differentiable, so the calibration
loss trains it by gradient descent.
CREDO further turns the trained confidence into a signal for accuracy.
A confident mistake or an unexpected success is more informative than a rollout
where confidence and outcome already agree, so CREDO weights learning on the
answer by the size of the discrepancy.
Accuracy and calibration improve together in both domains.

Our main contributions are as follows:
\begin{itemize}[leftmargin=*]
\item We show that sampling the confidence as text imposes two costs: the
report introduces variance and in practice collapses to a handful of distinct
values, and the sampling step makes the confidence non-differentiable, forcing
the calibration loss through a scalar reward.
\item We propose CREDO, which reads the confidence as the relative probability
between a pair of reserved tokens and trains it by differentiable regression,
giving the calibration loss a direct gradient path. The policy gradient
continues to handle correctness and credit assignment.
\item We turn the trained confidence into a signal for accuracy, weighting
each rollout by the gap between confidence and outcome. The weighting improves
accuracy in both domains.
\item Across mathematical and code reasoning, CREDO attains the best accuracy
and calibration. Ablations isolate the contributions of the readout and the
weighting, and the improved calibration pays off downstream in abstention and
selective prediction.
\end{itemize}

\section{Confidence channels and their costs}
\label{sec:channels}

We start from the RLVR setup and the calibration objective
(\S\ref{sec:setup}). The rest of the section identifies the channel choice
behind verbalized confidence (\S\ref{sec:onechoice}), analyzes its two costs
(\S\ref{sec:costs}), and derives three requirements for removing them
(\S\ref{sec:requirements}).

\subsection{Setup: RLVR and the calibration objective}
\label{sec:setup}

We work in the GRPO form of RLVR \citep{grpo}.
For a question $x$, the policy $\pi_\theta$ samples a group of $G$ rollouts
$o_1,\dots,o_G$, and a verifier scores each one for correctness,
$a_i=\mathbb{1}[y_i\equiv y^\ast]$, where $y_i$ is the answer extracted from
$o_i$ and $y^\ast$ the reference answer. The policy maximizes the clipped
objective of \citet{ppo}, using the group as the baseline:
\begin{equation}
\label{eq:grpo}
\mathcal{J}(\theta)=\mathbb{E}\Big[\tfrac1G\sum_{i=1}^{G}\tfrac{1}{|o_i|}
\sum_{k\in o_i}\min\big(\rho_{i,k}\hat A_{i,k},\
\mathrm{clip}(\rho_{i,k},1\pm\epsilon)\,\hat A_{i,k}\big)\Big],
\quad \hat A_{i,k}=\hat A_i=\frac{r_i-\bar r}{\mathrm{std}(r)}.
\end{equation}
Here $\rho_{i,k}$ is the token-level importance ratio, $\epsilon$ the clipping
range, and $\bar r$ and $\mathrm{std}(r)$ the mean and standard deviation of the group's
rewards. The reward is the correctness score, $r_i=a_i$. In standard GRPO every
token of a rollout shares the same $\hat A_i$.
The verifier's signal is not differentiable in the parameters, so training on
it runs through the policy gradient.

A calibrated model must also report a scalar $q\in[0,1]$ alongside its answer,
and the report must be right as often as it claims:
\begin{equation}
\label{eq:cal}
\mathbb{E}[a\mid q=v]=v,\qquad\forall v\in[0,1].
\end{equation}
The standard way to train toward~\eqref{eq:cal} is a strictly proper scoring
rule, whose expectation is uniquely minimized at $\mathbb{E}[a\mid q]$
\citep{gneiting2007}; the Brier score $(q-a)^2$ is the canonical example
\citep{brier1950}.
RLVR on its own does not target~\eqref{eq:cal}. Its reward depends only on correctness and is indifferent to the model's
confidence, so nothing in the objective pushes a confident error down \citep{rlcr}.

\subsection{The verbalized channel}
\label{sec:onechoice}

A confidence report has two design choices: where the confidence comes from
and how it is trained.

\begin{definition}[confidence channel]
\label{def:channel}
A \emph{confidence channel} is a pair $(\emission,\trainpath)$: an emission map
$\emission$ from the rollout prefix to a scalar report, and a training path
$\trainpath$, the route by which the calibration objective reaches the
report's value.
\end{definition}

The methods that train a model to state its confidence inside RLVR make the
same pair of choices. The confidence is written out as a numeral in the text,
\begin{equation}
\label{eq:verb}
\emission^{\mathrm{verb}}:\quad \hat q=\mathrm{parse}(s),
\qquad s\sim\pi_\theta(\cdot\mid h),
\end{equation}
where $h=(x,y,\dots)$ is the rollout prefix up to the confidence slot. The
model generates a rollout containing a solution $y$ to $x$, samples a numeral
$s$ (say ``0.9'') at the confidence slot, and parses it into a
report $\hat q$. We call a channel whose emission includes a sampling step a
\emph{verbalized channel}.

This choice of $\emission$ determines $\trainpath$. Sampling makes the
realized value of $\hat q$ non-differentiable in the parameters, so the
calibration signal can only enter~\eqref{eq:grpo} as a scalar reward.
RLCR \citep{rlcr} and DCPO \citep{dcpo} are the two representative methods
along this route,
\begin{equation}
\label{eq:rewards}
\begin{aligned}
r^{\mathrm{RLCR}}&=a-(\hat q-a)^2,\\[2pt]
r^{\mathrm{DCPO}}_{\mathrm{reason}}&=a,\qquad
r^{\mathrm{DCPO}}_{\mathrm{conf}}=-\big|\hat q-R_{IG}\big|,\qquad
R_{IG}=\lambda\tilde R_G+(1-\lambda)a,
\end{aligned}
\end{equation}
where $\tilde R_G$ is the group's mean accuracy and $\lambda$ its
mixing weight.
The first couples the calibration term into a single reward; the second scores
the two segments separately and mixes the target. The disagreements are all
internal to the reward, but the channel is shared: every verbalized variant
uses the same emission and training path
\citep{sayself,rewardingdoubt,leng2025,lacie2024,lovec2026}.

In both analyses, the optimal report is calibrated and the calibration term
does not hurt accuracy \citep{rlcr,dcpo}.
These analyses validate the reward design but leave the channel unexamined.

\subsection{Costs of the verbalized channel}
\label{sec:costs}

Each component of the verbalized channel carries a cost.
The first falls on the emission $\emission$
(proofs of all propositions are in Appendix~\ref{app:proofs}).

\begin{proposition}[a verbalized report pays a variance tax or determinizes]
\label{prop:variance}
Fix a rollout prefix and let $\hat q$ be the parsed report of~\eqref{eq:verb},
with conditional mean $\bar q$ and variance $v$. Every verbalized policy either
pays a per-prefix tax that persists as long as the report is stochastic, or
determinizes its per-prefix report distribution. Under the squared loss the tax
is exactly $v$, since $\mathbb{E}[(\hat q-t)^2]=(\bar q-t)^2+v$ for any target
$t\in[0,1]$; under any strictly convex loss $\ell$, Jensen's inequality gives
$\mathbb{E}[\ell(\hat q,t)]-\ell(\bar q,t)>0$.
\end{proposition}

Proposition~\ref{prop:variance} leaves two outcomes: variance or
determinization. Determinization avoids the tax but does not by itself prevent
different prefixes from using different values.
Empirically, however, reports cluster onto a handful of values,
so correct and incorrect answers share the same confidence
(Figure~\ref{fig:collapse}). The collapse is visible already in an untrained
base model, and training does not spread the values out: GRPO, RLCR, and DCPO
all have effective levels in single digits.

\begin{figure}[t]
\centering
\includegraphics[width=\textwidth]{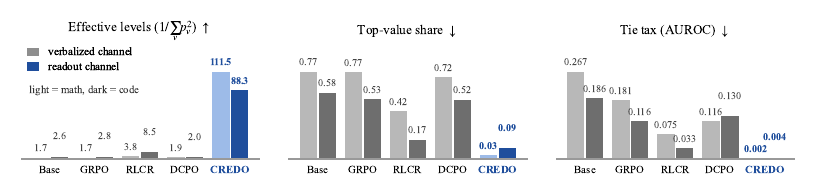}
\caption{Collapse of the verbalized channel, with CREDO for contrast. Effective levels $=1/\sum_v p_v^2$ (equiprobable-level
equivalent); top-value share = pool fraction of the most frequent value; tie
tax = AUROC lost to exactly equal reports.}
\label{fig:collapse}
\end{figure}

The variance can be avoided if the report is not sampled but computed. We call
a channel whose report is a
deterministic functional of the output distribution a \emph{readout channel}
(the rightmost group of Figure~\ref{fig:collapse}).
Even the conditional mean of a verbalized channel is such a functional and
removes the variance, but it cannot be optimized through the reward, which only
observes sampled values. Removing the variance is not enough; the
calibration loss must also reach the report directly.

The second cost falls on $\trainpath$: how the calibration loss trains the
report.

\begin{proposition}[gradient availability]
\label{prop:gradient}
Fix a calibration loss $\ell$. A sampled report admits no unbiased pathwise
gradient estimator, so $\ell$ can be queried only through its value. If instead
the report $c$ is a differentiable function of the parameters at a single
position, the exact derivative $\frac{\partial\ell}{\partial c}\nabla_\theta c$
is available, with zero estimator variance given the rollout. Exact pathwise
training also requires the report to be computable inside the forward pass that
generates the rollout. A verbalized numeral does not satisfy this: its realized
value is a sample, and its conditional mean requires marginalizing over the
numeral span.
\end{proposition}

For the same loss $\ell$, the two kinds of report give gradient estimators
\begin{equation}
\label{eq:estimators}
\underbrace{\ \hat g_{\mathrm{SF}}=\big(\ell(\hat q,t)-b\big)\,
\nabla_\theta\log \pi_\theta(\hat q\mid \cdot)\ }
_{\text{sampled report: zero-order query}}
\qquad\text{vs.}\qquad
\underbrace{\ \hat g_{\mathrm{PW}}=\frac{\partial \ell(c,t)}{\partial c}\,
\nabla_\theta c\ }
_{\text{deterministic report: exact derivative}}
\end{equation}
where $b$ is a baseline. The left side is the score-function estimator of
\citet{williams1992}; the right is the pathwise form, available when the report
is differentiable in the parameters \citep{kingma2014,mohamed2020}.

The two costs reinforce each other: once the grid collapses, the channel can no
longer express fine-grained values, and the scalar reward provides no direct gradient to recover them.

\subsection{Design requirements}
\label{sec:requirements}

The two costs come from a single choice. A report is either sampled, and then
the calibration signal reaches it only through a score-function estimate, or
computed deterministically, and then the calibration gradient reaches it
directly.

The two costs then convert into three requirements, two on $\emission$ and one
on $\trainpath$:
\begin{enumerate}[label=(\roman*),ref=(\roman*),leftmargin=*]
\item\label{req:functional} the report is a deterministic functional of the
output distribution rather than a sample from it, removing the variance term of
Proposition~\ref{prop:variance};
\item\label{req:grid} its value is not confined to the vocabulary grid,
removing the tie tax of Figure~\ref{fig:collapse};
\item\label{req:gradient} the calibration loss is differentiable in the
report, and the report is computable exactly in the forward pass that
generates the rollout, which yields the right-hand branch
of~\eqref{eq:estimators}.
\end{enumerate}
A pair of reserved tokens and their relative probability is a simple
functional of the output distribution that meets all three. The next section
develops such a channel and trains its report directly.

\section{CREDO: instantiating the readout channel}
\label{sec:credo}

CREDO operates within the RLVR loop of \S\ref{sec:setup}
(Figure~\ref{fig:method}), changing the confidence channel. It takes the
report from a pair of reserved tokens (\S\ref{sec:readout}) and trains it by a
regression that uses the calibration gradient directly (\S\ref{sec:calloss}),
meeting the three requirements of \S\ref{sec:requirements}. The policy gradient keeps
the verifier's correctness judgment and the credit assignment over the text
(\S\ref{sec:policyloss}), and the trained readout weights learning on the
answer (\S\ref{sec:sw}).

\begin{figure}[t]
\centering
\includegraphics[width=0.9\textwidth]{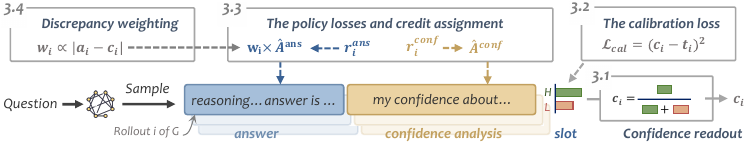}
\caption{Overview of CREDO. Each rollout contains an answer,
a confidence analysis, and a confidence slot. The readout~$c$ is the
relative probability of two reserved tokens at the slot. The calibration
loss trains~$c$ by differentiable regression toward a hybrid target~$t$.
The policy gradient scores the answer and analysis segments through
their text. Discrepancy weighting gives more weight to rollouts where
the outcome and the readout disagree.}
\label{fig:method}
\end{figure}

\subsection{Rollout format and the readout}
\label{sec:readout}

A rollout ends in a fixed format with three segments: an answer, a confidence
analysis, and the confidence slot.
At the first generated position inside the slot, a pair of reserved tokens
$\{H,L\}$ set aside in the vocabulary serves as the two events of the readout
\citep{selfref}.
Writing $h'$ for the prefix up to the slot, we
restrict the next-token distribution to the pair and renormalize,
\begin{equation}
\label{eq:readout}
c\;=\;\frac{\pi_\theta(H\mid h')}{\pi_\theta(H\mid h')+\pi_\theta(L\mid h')}
\;=\;\sigma(z_H-z_L)\ \in(0,1),
\end{equation}
where $z_H$ and $z_L$ are the logits of the two tokens and $\sigma$ is the
logistic function.
$c$ depends on the output distribution alone and varies continuously, which
gives~\ref{req:functional} and~\ref{req:grid}.
Because the readout is finished inside the forward pass that generates the
rollout, the second half of~\ref{req:gradient} holds as well. The training in \S\ref{sec:calloss} gives
$c$ its meaning as a confidence.

\subsection{The calibration loss}
\label{sec:calloss}

The readout is differentiable, which gives the first half
of~\ref{req:gradient}. We train it by a regression whose target takes the
within-group hybrid form of \citet{dcpo},
\begin{equation}
\label{eq:calloss}
\mathcal{L}_{\mathrm{cal}}=\frac{1}{|\mathcal{V}|}\sum_{i\in\mathcal{V}}
\big(c_i-t_i\big)^2,\qquad
t_i=\gamma\,a_i+(1-\gamma)\,\bar a_{\mathrm{group}},
\end{equation}
where $\mathcal{V}$ collects the rollouts whose readout is available,
$\gamma\in[0,1]$ is the mixing weight, and $\bar a_{\mathrm{group}}$ is the mean
correctness over all $G$ rollouts. The group term shrinks the binary target toward
the group mean, reducing target variance conditional on the question.
The gradient follows the right-hand branch of~\eqref{eq:estimators} into the
readout position:
$\nabla_\theta\mathcal{L}_{\mathrm{cal}}\propto
2(c-t)\,c(1-c)\,\nabla_\theta(z_H-z_L)$,
one term per sample, with no estimator variance given the rollout.

\begin{proposition}[the Bayes-optimal readout]
\label{prop:fixedpoint}
Fix the rollout policy and assume independent rollouts given~$x$.
Let $\mathcal{I}$ comprise $x$ and this rollout's prefix up to the readout,
and $\mu=\mathbb{E}[a\mid\mathcal{I}]$ the conditional correctness probability.
The Bayes-optimal readout for the unfiltered per-rollout risk of~\eqref{eq:calloss} is
$c^\ast=\lambda\mu+(1-\lambda)\,p(x)$,
$\lambda=\gamma+(1{-}\gamma)/G$, and $\mathbb{E}[c^\ast\mid x]=p(x)$
for every~$\gamma$.
\end{proposition}

Here $p(x)=\mathbb{E}_{\pi_\theta}[a\mid x]$ is prompt-level accuracy under the
current policy.
Decreasing $\gamma$ reduces target variance conditional on the question but
shrinks $c^\ast$ toward $p(x)$, biasing it relative to~$\mu$.
Prompt-level mean consistency does not, in general, imply calibration as
defined in~\eqref{eq:cal}, so we treat the hybrid target as regularized
correctness supervision. \S\ref{sec:attribution} evaluates sensitivity to~$\gamma$;
Appendix~\ref{app:proofs} contains the proof and quantifies target variance and shrinkage.

\subsection{Policy losses and credit assignment}
\label{sec:policyloss}

With the report's value trained by the regression, the policy gradient is left
with two jobs: whether the answer is right, and whether the confidence
analysis deserves reinforcement.
The answer and analysis segments are scored separately, on the token sets
$\mathrm{ans}_i$ and $\mathrm{conf}_i$, and $F_i\in\{0,1\}$ records whether all
three segments are present and parseable,
\begin{equation}
\label{eq:rewards3}
r^{\mathrm{ans}}_i\;=\;\beta\,F_i+a_i,\qquad
r^{\mathrm{conf}}_i\;=\;F_i\,\big(1-(t_i-c_i)^2\big),
\end{equation}
where $\beta$ weights the format term.
Each reward is standardized within its own group into $\hat A^{\mathrm{ans}}$
and $\hat A^{\mathrm{conf}}$, each acting only on its own segment; the segmented
credit assignment is inherited from \citet{dcpo}.
The policy gradient thus scores the two text segments, while the slot's value
is trained by the exact derivative of $\mathcal{L}_{\mathrm{cal}}$; the policy gradient and the calibration loss
share parameters but no credit.

\subsection{Discrepancy weighting: the readout as a training signal}
\label{sec:sw}

A rollout whose outcome falls far from what the model expected has more to teach
about the answer.
We therefore turn that distance into a per-sample weight.
\emph{Discrepancy weighting} centers it within the group and lets the
high-discrepancy rollouts count for more in the answer segment,
\begin{equation}
\label{eq:sw}
s_i=\lvert a_i-c_i\rvert,\qquad
w_i=\mathrm{clip}\big(1+\kappa\,(s_i-\bar s_{\mathrm{group}}),\ w_-,\ w_+\big),
\qquad \hat A^{\mathrm{ans}}_i\;\leftarrow\;w_i\,\hat A^{\mathrm{ans}}_i,
\end{equation}
where $s_i=|a_i-c_i|$ is the discrepancy, $\kappa\ge 0$ sets the strength,
$\bar s_{\mathrm{group}}$ is its group mean, and $w_\pm$ are clipping bounds
with $w_->0$.
The standardization uses the unweighted rewards, and the weights have group
mean one before clipping, so the weighting redistributes learning within the group.
Calibration information thus becomes a per-sample signal inside the training
loop.

The training objective is
\begin{equation}
\label{eq:total}
\mathcal{L}\;=\;-\,\mathcal{J}\;+\;\alpha\,\mathcal{L}_{\mathrm{cal}},\qquad
\hat A_{i,k}\;=\;\mathbb{1}[k\in\mathrm{ans}_i]\,w_i\hat A^{\mathrm{ans}}_i
\;+\;\mathbb{1}[k\in\mathrm{conf}_i]\,\hat A^{\mathrm{conf}}_i,
\end{equation}
where $\mathcal{J}$ is the objective of~\eqref{eq:grpo}, $\alpha$ weights the
calibration loss, and the per-token advantage $\hat A_{i,k}$
of~\eqref{eq:grpo} now takes the form on the right.

\section{Experiments}
\label{sec:experiments}

We compare the methods in and out of domain
(\S\ref{sec:mainresults}), locate the gain by removing one component at a time
(\S\ref{sec:attribution}), and measure what the calibration is worth to a user
who can abstain (\S\ref{sec:downstream}).

\subsection{Experimental setup}
\label{sec:setup-exp}

\textbf{Datasets and metrics.}
Mathematics trains on DeepScaleR \citep{deepscaler2025} and is evaluated on
seven sets: its in-distribution split, MATH-500
\citep{hendrycks2021measuring,lightman2024lets}, AIME 2024--2026, and AMC
2023--2024. Code trains on DeepCoder \citep{deepcoder2025} and is evaluated on
four: its in-distribution split, LiveCodeBench v5 and v6
\citep{jain2025livecodebench}, and HumanEval+ \citep{liu2023your}. We report
accuracy and three calibration measures: ECE \citep{guo2017} on the
reported values, AUROC on the ranking they induce, and Brier
\citep{brier1950} on both.

\textbf{Baselines and training.}
On Qwen3-8B \citep{qwen3} we compare the untrained base model with GRPO
\citep{grpo}, which trains no calibration term, and with RLCR \citep{rlcr} and
DCPO \citep{dcpo}, which train one. Training runs for two
epochs at learning rate $2\times10^{-6}$ with eight rollouts per question and
temperature 1.0. Every run uses seed 43, and the main comparisons add seeds 44
and 45.

Appendix~\ref{app:setup} provides prompt templates, training hyperparameters, the evaluation protocol, and other implementation details.

\subsection{Accuracy and calibration in and out of domain}
\label{sec:mainresults}
\begin{table}[t]
\centering
\setlength{\tabcolsep}{4pt}
\scriptsize
\setlength{\fboxsep}{1.5pt}
\caption{Accuracy and calibration in and out of domain, each entry
macro-averaged over evaluation sets and training seeds. Column groups are the evaluation domain. Baselines are read through their
own verbalized numeral and again through a digit expectation (Appendix~\ref{app:digitexp}), in the
rows \colorbox{digitbg}{named \texttt{-d}}; \colorbox{oursbg}{CREDO} is read
through its readout. \textbf{Bold} marks the best mean in each column and any
within one joint standard error of it.}
\label{tab:main}
\vspace{0.4em}
\begin{tabular}{llrrrrrrrr}
\toprule
& & \multicolumn{4}{c}{Mathematics} & \multicolumn{4}{c}{Code}\\
\cmidrule(lr){3-6}\cmidrule(lr){7-10}
& Method & Acc\,$\uparrow$ & ECE\,$\downarrow$ & AUROC\,$\uparrow$ & Brier\,$\downarrow$
& Acc\,$\uparrow$ & ECE\,$\downarrow$ & AUROC\,$\uparrow$ & Brier\,$\downarrow$\\
\midrule
& Base   & .439\sd{.001} & .472\sd{.002} & .644\sd{.004} & .453\sd{.002} & .455\sd{.001} & .418\sd{.001} & .690\sd{.004} & .395\sd{.002}\\
& GRPO   & .706\sd{.009} & .199\sd{.027} & .775\sd{.062} & .204\sd{.029} & .581\sd{.007} & .182\sd{.070} & .753\sd{.061} & .205\sd{.066}\\
& RLCR   & .682\sd{.010} & .159\sd{.012} & .831\sd{.009} & .171\sd{.009} & .553\sd{.013} & \textbf{.075}\sd{.005} & .815\sd{.014} & .146\sd{.006}\\
& DCPO   & .722\sd{.008} & .217\sd{.008} & .848\sd{.012} & .198\sd{.003} & .576\sd{.003} & .153\sd{.024} & .776\sd{.025} & .153\sd{.024}\\
\rowcolor{digitbg} & Base-d & .439\sd{.001} & .442\sd{.001} & .765\sd{.010} & .424\sd{.001} & .455\sd{.001} & .389\sd{.002} & .782\sd{.003} & .373\sd{.002}\\
\rowcolor{digitbg} & GRPO-d & .706\sd{.009} & .192\sd{.025} & .864\sd{.020} & .203\sd{.023} & .581\sd{.007} & .155\sd{.022} & .833\sd{.034} & .184\sd{.025}\\
\rowcolor{digitbg} & RLCR-d & .682\sd{.010} & .163\sd{.005} & .852\sd{.006} & .171\sd{.009} & .553\sd{.013} & .080\sd{.004} & .818\sd{.011} & .145\sd{.005}\\
\rowcolor{digitbg} & DCPO-d & .722\sd{.008} & .188\sd{.010} & .858\sd{.006} & .191\sd{.013} & .576\sd{.003} & .152\sd{.025} & \textbf{.849}\sd{.027} & .152\sd{.025}\\
\rowcolor{oursbg} \multirow{-9}{*}{\textbf{\textsc{In domain}}} & CREDO & \textbf{.752}\sd{.013} & \textbf{.103}\sd{.013} & \textbf{.933}\sd{.007} & \textbf{.086}\sd{.004} & \textbf{.600}\sd{.018} & \textbf{.064}\sd{.020} & \textbf{.850}\sd{.027} & \textbf{.127}\sd{.016}\\
\midrule
& GRPO   & \textbf{.564}\sd{.003} & .335\sd{.019} & .734\sd{.030} & .325\sd{.020} & .495\sd{.004} & .315\sd{.041} & .721\sd{.032} & .308\sd{.040}\\
& RLCR   & .515\sd{.011} & .170\sd{.004} & .815\sd{.016} & .186\sd{.007} & \textbf{.502}\sd{.002} & .228\sd{.014} & .792\sd{.005} & .229\sd{.010}\\
& DCPO   & .548\sd{.003}& .205\sd{.021} & .755\sd{.024} & .205\sd{.020} & \textbf{.508}\sd{.014} & .330\sd{.040} & .746\sd{.038} & .318\sd{.032}\\
\rowcolor{digitbg} & GRPO-d & \textbf{.564}\sd{.003} & .311\sd{.017} & .826\sd{.004} & .308\sd{.015} & .495\sd{.004} & .298\sd{.028} & .792\sd{.026} & .299\sd{.028}\\
\rowcolor{digitbg} & RLCR-d & .515\sd{.011} & .172\sd{.006} & .833\sd{.014} & .184\sd{.006} & \textbf{.502}\sd{.002} & .232\sd{.013} & .816\sd{.003} & .226\sd{.011}\\
\rowcolor{digitbg} & DCPO-d & .548\sd{.003} & .203\sd{.020} & .835\sd{.004} & .203\sd{.020} & \textbf{.508}\sd{.014} & .275\sd{.027} & .778\sd{.021} & .283\sd{.017}\\
\rowcolor{oursbg} \multirow{-7}{*}{\textbf{\textsc{Out of domain}}} & CREDO & \textbf{.560}\sd{.007} & \textbf{.117}\sd{.030} & \textbf{.855}\sd{.013} & \textbf{.142}\sd{.013} & \textbf{.504}\sd{.021} & \textbf{.138}\sd{.016} & \textbf{.836}\sd{.009} & \textbf{.165}\sd{.008}\\
\bottomrule
\end{tabular}
\end{table}

Table~\ref{tab:main} compares each method three ways: in its own domain, with
only the readout replaced, and on the other domain with nothing adapted.

In domain, CREDO leads on all four metrics in both mathematics and code. Among the
baselines, DCPO is the most accurate on mathematics but still trails CREDO on calibration;
RLCR trades accuracy for a tighter ECE on code, where it is the only baseline to match CREDO on any calibration metric.
All baselines use the verbalized channel and incur both costs of
\S\ref{sec:costs} (Figure~\ref{fig:collapse}).

The \texttt{-d} rows replace the sampled numeral with the expectation over
each baseline's digit distribution, giving a continuous confidence without
retraining. Every swap improves ranking metrics such as AUROC,
confirming that sampling adds noise to an otherwise informative signal. None of
the swaps closes the calibration gap to CREDO: the reported values were shaped
by the training path, and changing the emission alone does not recalibrate them.

The \textsc{out of domain} rows move each checkpoint to the other domain.
Accuracy is comparable across methods, but CREDO's calibration
advantage persists and remains the largest.
RLCR was the only baseline to match CREDO's code ECE in domain; that match
does not survive the domain shift. The calibration advantage transfers across domains.

Appendix~\ref{app:scales} repeats the mathematics comparison at 1.7B and
4B. CREDO leads on all four metrics at every scale, and its calibration
advantage widens as models grow larger. Per-set breakdowns and calibration
curves appear in Appendices~\ref{app:results} and~\ref{app:calibcurves}.

\subsection{Ablation and attribution}
\label{sec:attribution}
\begin{table}[t]
\centering
\setlength{\tabcolsep}{4pt}
\scriptsize
\caption{Leave-one-out ablation of CREDO (seed 43).}
\label{tab:ablation}
\vspace{0.4em}
\begin{tabular}{lrrrrrrrr}
\toprule
& \multicolumn{4}{c}{Mathematics} & \multicolumn{4}{c}{Code}\\
\cmidrule(lr){2-5}\cmidrule(lr){6-9}
& Acc\,$\uparrow$ & ECE\,$\downarrow$ & AUROC\,$\uparrow$ & Brier\,$\downarrow$
& Acc\,$\uparrow$ & ECE\,$\downarrow$ & AUROC\,$\uparrow$ & Brier\,$\downarrow$\\
\midrule
\rowcolor{oursbg} CREDO & .753 & .094 & .933 & .088 & .614 & .052 & .869 & .117\\
w/o differentiable regression ($\alpha{=}0$) & .696 & .241 & .880 & .232 & .611 & .172 & .857 & .172\\
w/o uncertainty-analysis segment & .709 & .096 & .902 & .107 & .587 & .048 & .868 & .117\\
w/o confidence-segment advantage ($\hat A^{\mathrm{conf}}{\equiv}0$) & .696 & .099 & .909 & .102 & .576 & .047 & .867 & .126\\
w/o discrepancy weighting ($\kappa{=}0$) & .721 & .122 & .904 & .102 & .572 & .051 & .843 & .132\\
\bottomrule
\end{tabular}
\end{table}

\begin{figure}[t]
\centering
\includegraphics[width=\textwidth]{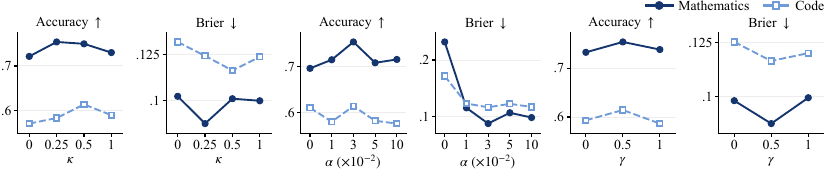}
\caption{Accuracy and Brier for the two evaluation domains as each
coefficient ($\kappa$, $\alpha$, $\gamma$) of \S\ref{sec:credo} is varied on seed 43, the others held fixed.
Exact values are in Table~\ref{tab:sens}.}
\label{fig:sensitivity}
\end{figure}

Table~\ref{tab:ablation} removes one component at a time. Setting $\alpha{=}0$
disables the calibration loss entirely: ECE and Brier degrade sharply in both
domains, and accuracy drops on mathematics. Without differentiable regression
the readout loses its direct supervision and calibration collapses. The next two rows
restore the regression but remove the confidence reasoning: dropping the
analysis text, or keeping it but withholding its policy-gradient credit, both
lower accuracy in both domains while calibration is largely unaffected. This
suggests that reasoning about confidence improves the model's answers, not only
its confidence estimates. Appendix~\ref{app:decompose} goes further,
independently swapping the emission and training path; neither axis alone
recovers the full method's performance.

Discrepancy weighting is the only component feeding the readout back into
the answer segment. Removing it causes the largest accuracy drop on code and a
substantial drop on mathematics. Calibration also degrades on mathematics,
though the weighting acts only on the answer advantage. The confidence and
answer channels are therefore not independent: the readout
shapes how the model learns from its answers.
Appendix~\ref{app:sw_baseline} adds the same weighting to RLCR and DCPO,
driven by the parsed numeral or a continuous digit expectation; in
neither case are the gains as large or as consistent as CREDO's.

Figure~\ref{fig:sensitivity} varies each coefficient in turn. Any moderate
discrepancy weight $\kappa$ outperforms $\kappa{=}0$, confirming that the
weighting itself, not just the readout, contributes to accuracy. The calibration
weight $\alpha$ has a threshold effect: $\alpha{=}0$ disables calibration
training and both ECE and Brier collapse, while all tested positive values restore
calibration. The target mixture $\gamma$ trades variance against shrinkage (\S\ref{sec:calloss});
among tested values, $\gamma{=}0.5$ gives the best accuracy and Brier in both domains, though $\gamma{=}1$ achieves lower code ECE.
Appendix~\ref{app:rewardshape} compares two alternative reward shapes;
the additive form is the best or tied in mathematics and competitive in code.

\subsection{Downstream benefits of calibration}
\label{sec:downstream}

\begin{figure}[t]
\centering
\begin{minipage}[t]{0.49\textwidth}
\vspace{0pt}
\centering
\captionof{table}{AURC, selective accuracy at three coverages, and the number of
operable thresholds.}
\label{tab:riskcov}
\vspace{0.4em}
\scriptsize
\renewcommand{\arraystretch}{1.12}
\setlength{\tabcolsep}{2.5pt}
\begin{tabular}{lrrrrr}
\toprule
& AURC\,(\%)\,$\downarrow$ & acc@50\,$\uparrow$ & acc@80\,$\uparrow$ & acc@90\,$\uparrow$ & Thr.\,$\uparrow$\\
\midrule
\multicolumn{6}{l}{\textsc{Mathematics}}\\
\rowcolor{oursbg} CREDO & \textbf{2.62}\sd{0.04} & \textbf{.986}\sd{.001} & \textbf{.962}\sd{.001} & \textbf{.927}\sd{.002} & \textbf{372}\\
DCPO & 4.98\sd{0.65} & .958\sd{.006} & .941\sd{.002} & .909\sd{.005} & 22\\
RLCR & 5.88\sd{0.67} & .952\sd{.009} & .910\sd{.009} & .878\sd{.007} & 15\\
GRPO & 7.92\sd{0.67} & .927\sd{.003} & .915\sd{.023} & .888\sd{.011} & 21\\
Base & 29.64\sd{0.30} & .712\sd{.001} & .705\sd{.005} & .666\sd{.001} & 34\\
\midrule
\multicolumn{6}{l}{\textsc{Code}}\\
\rowcolor{oursbg} CREDO & \textbf{16.80}\sd{1.57} & \textbf{.885}\sd{.013} & \textbf{.686}\sd{.025} & \textbf{.617}\sd{.019} & \textbf{932}\\
DCPO & 21.26\sd{0.95} & .863\sd{.016} & \textbf{.679}\sd{.003} & \textbf{.611}\sd{.003} & 10\\
RLCR & 21.36\sd{1.00} & .809\sd{.015} & .608\sd{.014} & .562\sd{.014} & 12\\
GRPO & 21.50\sd{3.08} & .797\sd{.081} & .652\sd{.041} & .599\sd{.023} & 17\\
Base & 45.10\sd{0.50} & .572\sd{.005} & .478\sd{.002} & .437\sd{.002} & 16\\
\bottomrule
\end{tabular}
\end{minipage}\hfill
\begin{minipage}[t]{0.49\textwidth}
\vspace{0pt}
\centering
\includegraphics[width=\textwidth]{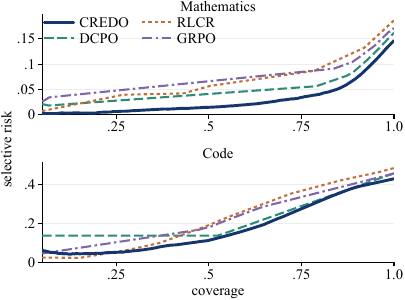}
\caption{Risk--coverage frontiers, pooled over three seeds; Base omitted.}
\label{fig:riskcov}
\end{minipage}
\end{figure}

Figure~\ref{fig:riskcov} plots the error among answered questions as the model
keeps only its most confident fraction \citep{geifman2017}. In mathematics CREDO
is below every baseline throughout. In code RLCR is lower at very low coverage; above
it CREDO is lowest. At the evaluated coverage budgets
(Table~\ref{tab:riskcov}), CREDO answers the most accurately and has the lowest
aggregate error in both domains; only the two tightest code budgets are
statistically tied with DCPO.
Selective prediction relies on setting a confidence threshold, and the channel's granularity determines how finely it can be tuned.
Verbalized reports concentrate on a few values
(Figure~\ref{fig:collapse}), leaving a threshold policy only a dozen or two
operating points (Table~\ref{tab:riskcov}); the readout supplies hundreds. On
code, DCPO's most confident value already covers half the pool.

\section{Related work}
\label{sec:related}

\paragraph{Confidence estimation.}
The most direct way to obtain a confidence estimate is to ask for one:
\citet{lin2022teaching} fine-tune a model to state a number, but such numbers
are systematically overconfident \citep{xiong2024llms,tian2023just}.
Instead of asking, the confidence can be read from the model's own probabilities
or internal representations
\citep{fomicheva2020,kadavath2022language,azaria2023internal,orgad2025llms}, or
estimated from sample agreement
\citep{kuhn2023semantic,aichberger2025improving,geng2024survey}.
A dedicated readout can also be trained directly: \citet{selfref} learn a
confidence-token readout under supervised finetuning, applied to routing
and rejection. CREDO trains the same form of readout on-policy inside
RLVR, against a hybrid calibration target, and feeds it back as a signal
for answer learning.

\paragraph{Calibration in RLVR.}
Inside RLVR, one line of work incorporates confidence derived from response
likelihood or token-level statistics into policy optimization
\citep{capo2026,ccgspg,egpo}. Another trains the model to state a numerical
confidence and scores it with a reward.
RLCR \citep{rlcr} folds a proper scoring rule into the reward, while
DCPO \citep{dcpo} separates credit by scoring answer and confidence tokens; other work extends to question answering, preference
tuning, and generative recommendation
\citep{sayself,rewardingdoubt,lacie2024,leng2025,ugr}.
In the numerical-report methods, the confidence is sampled as text and
trained through a scalar reward (\S\ref{sec:channels}).
A dedicated readout can also serve as a self-reward signal:
LaSeR \citep{laser} aligns a last-token score with verifier rewards
during RLVR and mixes score-derived advantages into policy updates.
CREDO targets calibration rather than self-reward: it regresses a
normalized two-token readout toward a hybrid calibration target and
uses confidence--outcome discrepancies to rescale answer advantages.

\section{Conclusion}
\label{sec:conclusion}

In this work, we examined the confidence channel shared by all verbalized
calibration methods. We analyzed the variance cost of sampling the report,
observed that it collapses to a handful of values in the baselines we
evaluated, and showed that sampling forces the calibration loss through a
scalar reward.
We therefore proposed CREDO, which reads the confidence as a deterministic
function of the output distribution and trains it by differentiable regression,
while the policy gradient optimizes correctness.
The same readout also serves as a training signal: it weights learning on the
answer by how far each outcome departs from the model's own confidence, and
this improves accuracy in both domains.
On mathematical and code reasoning tasks CREDO attains the best accuracy and
calibration, and its confidence translates into finer abstention control.
Because CREDO trains confidence by regression, it is not tied to binary
correctness labels; open-ended tasks with graded reward models are a natural
next step.
Looking further, feeding the confidence signal back during generation would let
uncertain steps be reconsidered before the answer is finalized, moving toward
reasoners that know when they are right.

\markmainend

\bibliography{references}
\bibliographystyle{iclr2027_conference}

\clearpage
\appendix
\section{Proofs}
\label{app:proofs}

For the verbalized channel, $h$ denotes the rollout prefix up to the confidence
slot, and the report is $\hat q=\mathrm{parse}(s)$ for a numeral
$s\sim\pi_\theta(\cdot\mid h)$ sampled from the policy, taking values in a
finite set $\mathcal{Q}\subset[0,1]$. Expectations are conditional on $h$ and
on the other rollouts of the group, which fixes any target $t\in[0,1]$
determined by them.

\subsection{Proof of Proposition~\texorpdfstring{\ref{prop:variance}}{1}}

\emph{Fix a rollout prefix and let $\hat q$ be the parsed report
of~\eqref{eq:verb}, with conditional mean $\bar q$ and variance $v$. Every
verbalized policy either pays a per-prefix tax that persists as long as the report is stochastic, or
determinizes its per-prefix report distribution.
Under the squared loss the tax is exactly $v$, since $\mathbb{E}[(\hat
q-t)^2]=(\bar q-t)^2+v$ for any target $t\in[0,1]$; under any strictly convex
loss $\ell$, Jensen's inequality gives
$\mathbb{E}[\ell(\hat q,t)]-\ell(\bar q,t)>0$.}

\begin{proof}
The tax is the excess of scoring the sampled report over scoring the same
policy's conditional mean, $\tau=\mathbb{E}[\ell(\hat q,t)]-\ell(\bar q,t)$.

Under the squared loss, $t$ and $\bar q$ are fixed given $h$, and expanding
about $\bar q$ gives
\begin{equation}
\mathbb{E}\big[(\hat q-t)^2\big]=\mathbb{E}\big[(\hat q-\bar q)^2\big]
+2(\bar q-t)\,\mathbb{E}\big[\hat q-\bar q\big]+(\bar q-t)^2=v+(\bar q-t)^2 ,
\end{equation}
where the middle term vanishes. So the tax is the conditional variance itself,
$\tau=v$, and it does not involve the target. Changing calibration moves
$\bar q$, and with it only the $(\bar q-t)^2$ term.

For a strictly convex $\ell$, Jensen's inequality gives $\mathbb{E}[\ell(\hat
q,t)]\ge\ell(\bar q,t)$, with equality only if $\hat q$ is almost surely
constant given $h$.

Both branches follow. A non-degenerate per-prefix report distribution pays
$\tau>0$. A degenerate one concentrates the report on a single element of
$\mathcal{Q}$ and pays nothing --- the determinized case.
\end{proof}

\subsection{Proof of Proposition~\texorpdfstring{\ref{prop:gradient}}{2}}

\emph{Fix a calibration loss $\ell$. A sampled report admits no unbiased pathwise
gradient estimator, so $\ell$ can be queried only through its value. If instead
the report $c$ is a differentiable function of the parameters at a single
position, the exact derivative $\frac{\partial\ell}{\partial c}\nabla_\theta c$
is available, with zero estimator variance given the rollout. Exact pathwise
training also requires the report to be computable inside the forward pass that
generates the rollout. A verbalized numeral does not satisfy this: its realized
value is a sample, and its conditional mean requires marginalizing over the
numeral span.}

\begin{proof}
A pathwise estimator \citep{kingma2014} writes the draw as $s=S_\theta(\xi)$,
with noise $\xi$
whose law does not depend on $\theta$, and differentiates the composition,
$\frac{\partial\ell}{\partial\hat q}\nabla_\theta\,\mathrm{parse}(S_\theta(\xi))$.
The map $\theta\mapsto\mathrm{parse}(S_\theta(\xi))$ takes values in the finite
set $\mathcal{Q}$, so it is locally constant wherever it is continuous and its
gradient is zero wherever it exists. The estimator is therefore zero, while the
gradient
\begin{equation}
\nabla_\theta\,\mathbb{E}\big[\ell(\hat q,t)\big]
=\sum_{s}\ell\big(\mathrm{parse}(s),t\big)\,\nabla_\theta\pi_\theta(s\mid h)
\end{equation}
is not whenever two values in $\mathcal{Q}$ carry different losses and the
policy can shift probability between them. In this identity $\ell$ enters only
through its values on $\mathcal{Q}$, never through
$\partial\ell/\partial\hat q$. Equivalently, the same gradient has the
score-function form \citep{williams1992}
$\mathbb{E}[\ell(\hat q,t)\,\nabla_\theta\log\pi_\theta(s\mid h)]$, the
left-hand branch of~\eqref{eq:estimators}.

Now let $c$ be differentiable in $\theta$ at a single position, with $\ell$
differentiable in the report. Given the rollout, $c$ is a deterministic
function of $\theta$, so the chain-rule derivative
$\nabla_\theta\ell(c,t)=\frac{\partial\ell}{\partial c}\nabla_\theta c$ is
exact and carries no variance.

Exact pathwise training also needs a differentiable value inside the pass that
generates the rollout, and a verbalized numeral has none. Its realized value is
a sample, which carries no pathwise derivative. Its conditional mean,
\begin{equation}
\bar q=\sum_s\pi_\theta(s\mid h)\,\mathrm{parse}(s) ,
\end{equation}
is differentiable in $\theta$, but the sum ranges over every string the numeral
span can produce, because $\mathrm{parse}$ couples the span's positions. The
generating pass produces one of those strings, not the sum. A report read at a
single position is instead a function of that position's logits, which the same
pass has already computed.
\end{proof}

\subsection{Proof of Proposition~\texorpdfstring{\ref{prop:fixedpoint}}{3}}

\emph{Fix the rollout policy and assume independent rollouts given~$x$.
Let $\mathcal{I}$ comprise $x$ and this rollout's prefix up to the readout,
and $\mu=\mathbb{E}[a\mid\mathcal{I}]$ the conditional correctness probability.
The Bayes-optimal readout for the unfiltered per-rollout risk of~\eqref{eq:calloss} is
$c^\ast=\lambda\mu+(1-\lambda)\,p(x)$,
$\lambda=\gamma+(1{-}\gamma)/G$, and $\mathbb{E}[c^\ast\mid x]=p(x)$
for every~$\gamma$.}

\begin{proof}
We analyze the unfiltered per-rollout population risk
$R(c)=\mathbb{E}[(c(\mathcal{I})-t)^2]$, which coincides with the expected
loss of~\eqref{eq:calloss} when every rollout has an available readout.
Expectations are over the group drawn at a question $x$, and
$p(x)=\mathbb{E}[a\mid x]$ denotes the current policy's accuracy on it.
The target is $t=\gamma a+(1-\gamma)\bar a_{\mathrm{group}}$, with
$\bar a_{\mathrm{group}}$ the mean correctness over all $G$ rollouts. Each
squared error is minimized pointwise by the conditional mean of its target, so
the best readout among functions of $\mathcal{I}$ is
$c^\ast=\mathbb{E}[t\mid\mathcal{I}]$.

By definition $\mathbb{E}[a\mid\mathcal{I}]=\mu$. The other rollouts are drawn
independently given $x$, and $\mathcal{I}$ is a function of $x$ and this
rollout alone, so each of the other $G-1$ scores has conditional mean $p(x)$,
and
\begin{equation}
\mathbb{E}\big[\bar a_{\mathrm{group}}\mid\mathcal{I}\big]
=\frac{\mu}{G}+\frac{G-1}{G}\,p(x)=p(x)+\frac{\mu-p(x)}{G} .
\end{equation}
Therefore
\begin{equation}
c^\ast=\gamma\,\mu+(1-\gamma)\,p(x)+(1-\gamma)\,\frac{\mu-p(x)}{G}
=\lambda\mu+(1-\lambda)\,p(x),\qquad
\lambda=\gamma+\frac{1-\gamma}{G}.
\end{equation}
Averaging over rollouts at the same question, $\mathbb{E}[\mu\mid
x]=\mathbb{E}[\mathbb{E}[a\mid\mathcal{I}]\mid x]=p(x)$, so
$\mathbb{E}[c^\ast\mid x]=\lambda\,p(x)+(1-\lambda)\,p(x)=p(x)$ for
every~$\gamma$.
\end{proof}

\paragraph{Target variance and shrinkage bias.}

Writing $t_i=\lambda a_i+\frac{1-\gamma}{G}\sum_{j\ne i}a_j$ and using
conditional independence given~$x$:

\emph{Target variance.}
\begin{equation}
\operatorname{Var}(t_i\mid x)
=\left[\gamma^2+\frac{1-\gamma^2}{G}\right]p(x)\bigl(1-p(x)\bigr).
\end{equation}
At $G{=}8$ and $\gamma{=}0.5$ the bracketed coefficient is $0.344$, compared
with~$1$ at $\gamma{=}1$.
This is a reduction in marginal target variance conditional on the question;
it does not by itself establish lower prediction error or gradient variance.

\emph{Shrinkage bias.}  Relative to the conditional correctness
probability~$\mu$,
\begin{equation}
c^\ast_\gamma-\mu=(1-\lambda)\bigl(p(x)-\mu\bigr).
\end{equation}
For fixed $\gamma{<}1$, the limit as $G\to\infty$ is
$(1{-}\gamma)(p(x){-}\mu)$: increasing the group size does not generally
eliminate the bias.

\emph{Brier cost.}  Because $\mu$ minimizes Brier risk over
$\mathcal{I}$-measurable predictors and the cross term vanishes
($\mathbb{E}[a{-}\mu\mid\mathcal{I}]=0$),
\begin{equation}
\mathbb{E}\bigl[(c^\ast_\gamma-a)^2\bigr]
-\mathbb{E}\bigl[(\mu-a)^2\bigr]
=(1-\lambda)^2\,\mathbb{E}\bigl[(\mu-p(x))^2\bigr].
\end{equation}
The hybrid target thus reduces target variance conditional on~$x$ while
introducing a population-level Brier cost relative to~$\mu$.
Whether this regularization improves the learned readout under finite data
and limited optimization is an empirical question
(\S\ref{sec:attribution}).

\section{Experimental setup and implementation details}
\label{app:setup}

\subsection{Model and prompts}
\label{app:model}

All experiments use Qwen3-8B \citep{qwen3}. The reserved tokens $H$ and $L$ of
\S\ref{sec:readout} are added to the vocabulary as special tokens, with
embeddings copied from the rows of ``high'' and ``low''.

The system prompts of the two channels fix the three-segment format of
\S\ref{sec:readout}: they differ only in the slot instruction, and the coding
versions differ only in carrying the answer in a \texttt{python} block.

\begin{center}
\fcolorbox{promptedge}{promptbg}{%
\begin{minipage}{0.94\textwidth}
\small\raggedright
\textbf{Verbalized channel.}\ A conversation between User and Assistant. The
user asks a question, and the Assistant solves it. Solve the problem step by
step, and put your final answer within \texttt{\textbackslash boxed\{\}}. After
the final answer, analyze the uncertainty of your solution within
\texttt{<analysis>} \texttt{</analysis>} tags. This analysis is the basis for
the confidence level you will report next, and must follow these rules: (1)
point out specific steps that could be wrong or ambiguous, including alternative
approaches that might lead to different answers; (2) do not solve the problem
again and do not revise or change your answer; (3) be specific --- if you cannot
find more uncertainties, say so explicitly. Then provide your confidence that
the final answer is correct, as a decimal number between 0 and 1 (e.g.\ 0.3 or
0.8), within \texttt{<confidence>} \texttt{</confidence>} tags. The final format
that must be followed is: \{step-by-step solution with
\texttt{\textbackslash boxed\{final answer\}}\} \texttt{<analysis>} uncertainty
analysis here \texttt{</analysis>} \texttt{<confidence>} confidence here
\texttt{</confidence>}
\end{minipage}}
\end{center}

\begin{center}
\fcolorbox{promptedge}{promptbg}{%
\begin{minipage}{0.94\textwidth}
\small\raggedright
\textbf{Readout channel (CREDO).}\ A conversation between User and Assistant.
The user asks a question, and the Assistant solves it. Solve the problem step by
step, and put your final answer within \texttt{\textbackslash boxed\{\}}. After
the final answer, analyze the uncertainty of your solution within
\texttt{<analysis>} \texttt{</analysis>} tags. This analysis is the basis for
the confidence level you will report next, and must follow these rules: (1)
point out specific steps that could be wrong or ambiguous, including alternative
approaches that might lead to different answers; (2) do not solve the problem
again and do not revise or change your answer; (3) be specific --- if you cannot
find more uncertainties, say so explicitly. Then, within \texttt{<confidence>}
\texttt{</confidence>} tags, output exactly one token: \texttt{<CONF\_HIGH>} if
your final answer is more likely correct than not, otherwise
\texttt{<CONF\_LOW>}. The final format that must be followed is:
\{step-by-step solution with \texttt{\textbackslash boxed\{final answer\}}\}
\texttt{<analysis>} uncertainty analysis here \texttt{</analysis>}
\texttt{<confidence>}\texttt{<CONF\_HIGH>}\texttt{</confidence>}
\end{minipage}}
\end{center}

\subsection{Training}
\label{app:training}

\begin{center}
\begin{minipage}{\textwidth}
\centering\small
\captionof{table}{Training configuration. Values that differ by domain are given
as mathematics / code. Rollouts are generated with vLLM \citep{vllm}.}
\label{tab:hyper}
\vspace{0.4em}
\setlength{\tabcolsep}{5pt}
\begin{tabular}{@{}ll@{\hspace{2.0em}}ll@{}}
\toprule
\multicolumn{4}{@{}l}{\textsc{Shared}}\\
\midrule
Base model & Qwen3-8B & Precision & bf16\\
Seeds & $43$, $44$, $45$ & Rollouts per question $G$ & $8$\\
Learning rate & $2\times10^{-6}$ & Per-device batch size & $2$\\
Schedule & constant, no warmup & Gradient accumulation & $64$\\
Optimizer & AdamW \citep{adamw} & Epochs & $2$\\
Rollout temperature & $1.0$ & Rollout top-$k$ & $50$\\
Prompt budget & $1{,}024$ / $2{,}048$ & Clipping range $\epsilon$ & $0.2$\\
Completion budget & $8{,}192$ / $6{,}144$ & KL coefficient & $0$\\
\midrule
\multicolumn{4}{@{}l}{\textsc{CREDO}}\\
\midrule
Target mixture $\gamma$ & $0.5$ & Weighting strength $\kappa$ & $0.25$ / $0.50$\\
Calibration weight $\alpha$ & $0.03$ & Weight clip $[w_-,w_+]$ & $[0.5,\,2.0]$\\
Format weight $\beta$ & $0.5$ & Warmup for $\alpha$ and $\kappa$ & $30$ steps\\
\bottomrule
\end{tabular}
\end{minipage}
\end{center}

\subsection{Evaluation and metrics}
\label{app:data}

Mathematics trains on DeepScaleR, code on DeepCoder, and each domain is
evaluated on the sets of Table~\ref{tab:data}. We evaluate at temperature
$0.7$, generating each trained run under its own training seed; the base model
has no training seed, so its error bars come from the three evaluation seeds.

\begin{center}
\begin{minipage}{\textwidth}
\centering\small
\captionof{table}{Datasets. $n$ is the number of samples drawn per problem at
evaluation. AIME and AMC are the corresponding years' competition problems;
HumanEval+ strengthens the tests of HumanEval \citep{chen2021evaluating}.}
\label{tab:data}
\vspace{0.4em}
\setlength{\tabcolsep}{5pt}
\begin{tabular}{@{}lrc@{\hspace{2.4em}}lrc@{}}
\toprule
\multicolumn{3}{c}{Mathematics} & \multicolumn{3}{c}{Code}\\
\cmidrule(lr){1-3}\cmidrule(lr){4-6}
Set & Problems & $n$ & Set & Problems & $n$\\
\midrule
DeepScaleR (train) & $9{,}500$ & --- & DeepCoder (train) & $9{,}500$ & ---\\
DeepScaleR test & $500$ & 4 & DeepCoder test & $500$ & 4\\
MATH-500 & $500$ & 4 & HumanEval+ & $163$ & 4\\
AIME 2024 & $30$ & 8 & LiveCodeBench v5 & $279$ & 4\\
AIME 2025 & $30$ & 8 & LiveCodeBench v6 & $131$ & 4\\
AIME 2026 & $30$ & 8 & & &\\
AMC 2023 & $46$ & 8 & & &\\
AMC 2024 & $45$ & 8 & & &\\
\bottomrule
\end{tabular}
\end{minipage}
\end{center}

Write $\{(c_i,a_i)\}_{i=1}^{n}$ for an evaluation pool, with $c_i$ the reported
confidence and $a_i\in\{0,1\}$ the correctness, decided by symbolic equivalence
in mathematics \citep{mathverify} and by executing the program against the
problem's tests in code.

\paragraph{Digit-probability expectation.}
\label{app:digitexp}
The \texttt{-d} rows throughout the paper replace the
sampled numeral with a continuous value derived from the model's token
probabilities at the confidence slot. Given the verbalized format
\texttt{<confidence>0.\textit{d}</confidence>}, we teacher-force the prefix up
to the confidence position and read the logits at two slots:
\begin{itemize}[leftmargin=1.5em,itemsep=1pt,topsep=2pt]
\item \emph{Integer position} (candidates ``0'' and ``1''), renormalized:
$\hat{p}_0,\;\hat{p}_1$.
\item \emph{Decimal position} (candidates ``0'' through ``9''), renormalized:
$\hat{q}_d$ for $d=0,\dots,9$.
\end{itemize}
The confidence is a mixture over the two integer outcomes:
\begin{equation}
c \;=\; \hat{p}_1 \cdot 1.0 \;+\; \hat{p}_0 \cdot
\frac{\mathbb{E}[d]}{10},
\qquad
\mathbb{E}[d] = \sum_{d=0}^{9} d\,\hat{q}_d\,.
\label{eq:digitexp}
\end{equation}
The mixture covers $[0,1]$ continuously. Probabilities are renormalized
within each candidate set, and the logits are read from the existing rollout
without an additional forward pass.

We report the following metrics over each evaluation pool:
\begin{itemize}[leftmargin=1.5em,itemsep=3pt,topsep=2pt]
\item \textbf{Accuracy} is the mean of $a_i$.
\item \textbf{ECE} uses ten equal-width bins $B_1,\dots,B_{10}$
\citep{naeini2015,guo2017}, with $0$ and $1$
falling in the first and last bin:
\begin{equation}
\mathrm{ECE}=\sum_{b=1}^{10}\frac{|B_b|}{n}\,
\big|\overline{a}_{B_b}-\overline{c}_{B_b}\big|,
\end{equation}
where $\overline{a}_B$ and $\overline{c}_B$ are the bin's mean correctness and
mean confidence.
\item \textbf{Brier} is the mean squared error, $\frac1n\sum_i(c_i-a_i)^2$
\citep{brier1950}.
\item \textbf{AUROC} is the area under the ROC curve \citep{hanley1982} obtained
by scoring $a$ with $c$.
\item \textbf{Selective prediction} (AURC, acc@$\tau$, Thr.)
\citep{geifman2017} answers questions
in decreasing order of $c$. Equal confidences cannot be separated by a
threshold, so the achievable coverages are the boundaries of the tie groups:
for the distinct values $v_1>\dots>v_K$ of $c$,
\begin{equation}
\mathrm{cov}_k=\frac{|\{i:c_i\ge v_k\}|}{n},
\qquad
\mathrm{acc}_k=\frac{1}{|\{i:c_i\ge v_k\}|}\sum_{i:\,c_i\ge v_k}a_i ,
\end{equation}
and Thr.\ is the number $K$ of achievable points. AURC integrates the risk
$1-\mathrm{acc}$ against coverage over $[0,1]$ by the trapezoid rule, extended
as a constant below $\mathrm{cov}_1$, and $\mathrm{acc}@\tau$ interpolates
$(\mathrm{cov}_k,\mathrm{acc}_k)$ linearly in coverage.
\end{itemize}

\section{Additional analyses}
\label{app:analyses}

\subsection{Additional model scales}
\label{app:scales}

The main table uses Qwen3-8B. Table~\ref{tab:scales} runs the same mathematics
comparison at 1.7B and 4B (seed~43), with the 8B seed-43 run included for
direct comparison; three-seed averages for 8B appear in Table~\ref{tab:main}.
Accuracy grows with model size for every method. CREDO leads on all four
metrics at every scale, and its calibration advantage widens as models grow
larger. The digit-expectation rows (-d) consistently outrank their verbalized
counterparts in AUROC at all three scales, indicating that sampling noise is not
specific to 8B.

\noindent\begin{minipage}{\textwidth}
\centering\scriptsize
\setlength{\tabcolsep}{3pt}
\captionof{table}{Mathematics results at three model scales (seed 43). All
methods follow the same protocol as Table~\ref{tab:main}.
\colorbox{digitbg}{\texttt{-d}} rows evaluate through the digit expectation;
\colorbox{oursbg}{CREDO} through its logit readout.
\textbf{Bold} marks the best in each column.}
\label{tab:scales}
\vspace{0.4em}
\begin{tabular}{@{}l rrrr rrrr rrrr @{}}
\toprule
& \multicolumn{4}{c}{Qwen3-1.7B} & \multicolumn{4}{c}{Qwen3-4B} & \multicolumn{4}{c}{Qwen3-8B}\\
\cmidrule(lr){2-5}\cmidrule(lr){6-9}\cmidrule(lr){10-13}
& Acc\,$\uparrow$ & ECE\,$\downarrow$ & AUROC\,$\uparrow$ & Brier\,$\downarrow$
& Acc\,$\uparrow$ & ECE\,$\downarrow$ & AUROC\,$\uparrow$ & Brier\,$\downarrow$
& Acc\,$\uparrow$ & ECE\,$\downarrow$ & AUROC\,$\uparrow$ & Brier\,$\downarrow$\\
\midrule
Base       & .268 & .531 & .691 & .477  & .425 & .468 & .641 & .447  & .440 & .474 & .640 & .455 \\
GRPO       & .460 & .391 & .782 & .357  & .660 & .221 & .813 & .220  & .717 & .174 & .820 & .179 \\
RLCR       & .494 & .145 & .853 & .155  & .663 & .186 & .836 & .185  & .682 & .167 & .837 & .181 \\
DCPO       & .501 & .203 & .823 & .189  & .651 & .220 & .797 & .196  & .715 & .222 & .846 & .195 \\
\rowcolor{digitbg} Base-d  & .268 & .557 & .717 & .495  & .425 & .446 & .765 & .428  & .440 & .443 & .772 & .425 \\
\rowcolor{digitbg} GRPO-d  & .460 & .379 & .806 & .350  & .660 & .208 & .851 & .211  & .717 & .170 & .886 & .184 \\
\rowcolor{digitbg} RLCR-d  & .494 & .152 & .867 & .154  & .663 & .191 & .866 & .190  & .682 & .164 & .855 & .180 \\
\rowcolor{digitbg} DCPO-d  & .501 & .214 & .873 & .191  & .651 & .183 & .871 & .186  & .715 & .184 & .852 & .180 \\
\rowcolor{oursbg}  CREDO   & \textbf{.514} & \textbf{.115} & \textbf{.905} & \textbf{.122}  & \textbf{.689} & \textbf{.116} & \textbf{.884} & \textbf{.128}  & \textbf{.753} & \textbf{.094} & \textbf{.933} & \textbf{.088} \\
\bottomrule
\end{tabular}
\end{minipage}

\vspace{1.5em}

\subsection{Emission and training-path decomposition}
\label{app:decompose}

Table~\ref{tab:ablation} removes one component at a time while keeping the rest
of CREDO intact. To disentangle the emission from the training path, we fill
the two remaining cells of the $2{\times}2$ design space
(Table~\ref{tab:designspace}) and add one fully decoupled reference. All arms
follow the protocol of Table~\ref{tab:main}.

\begin{center}
\small
\captionof{table}{Design space for the confidence channel.}
\label{tab:designspace}
\vspace{0.4em}
\setlength{\tabcolsep}{6pt}
\begin{tabular}{@{}lcc@{}}
\toprule
& Score-function reward & Differentiable regression\\
\midrule
Verbalized numeral & RLCR / DCPO (Table~\ref{tab:main}) & Digit-exp.\ emission\\
$\{H,L\}$ readout  & Readout + reward only              & CREDO (Table~\ref{tab:main})\\
\bottomrule
\end{tabular}
\end{center}

\begin{itemize}[leftmargin=1.5em,itemsep=1pt,topsep=2pt]
\item \textbf{Digit-expectation emission.} The confidence slot produces a
verbalized numeral, but the value used for training and evaluation is the
continuous expectation over the digit probabilities rather than the \{H,L\}
logit readout. The regression loss, target, and discrepancy weighting are all
kept; only the emission differs from CREDO.
\item \textbf{Readout + reward only.} The \{H,L\} readout is kept, but
differentiable regression is removed and the readout tokens are trained through
the reward signal via policy gradient, as in verbalized methods. Only the
training path differs from CREDO.
\item \textbf{GRPO + auxiliary head.} A standard GRPO model whose answer reward
uses accuracy alone. The confidence channel plays no role in policy training; a
separate BCE head at the readout position learns to predict per-rollout
correctness.
\end{itemize}

Changing either axis alone degrades performance. Replacing the dedicated token
pair with a digit-probability expectation while keeping the regression causes
the largest drop: the verbalized digit distribution, though continuous, provides
a weaker training signal. Keeping the readout but training it only through the
reward also hurts, which confirms that differentiable regression is what drives
the calibration gains. A standalone BCE head on top of GRPO provides a natural
lower bound; it achieves reasonable calibration, but CREDO outperforms it on
every metric---optimizing the readout inside the training loop works better than
attaching a head afterward.

\noindent\begin{minipage}{\textwidth}
\centering\scriptsize
\setlength{\tabcolsep}{4pt}
\captionof{table}{Emission and training-path decomposition (seed 43). \textbf{Bold} marks the best in each column.}
\label{tab:decompose}
\vspace{0.4em}
\begin{tabular}{@{}lrrrrrrrr@{}}
\toprule
& \multicolumn{4}{c}{Mathematics} & \multicolumn{4}{c}{Code}\\
\cmidrule(lr){2-5}\cmidrule(lr){6-9}
& Acc\,$\uparrow$ & ECE\,$\downarrow$ & AUROC\,$\uparrow$ & Brier\,$\downarrow$
& Acc\,$\uparrow$ & ECE\,$\downarrow$ & AUROC\,$\uparrow$ & Brier\,$\downarrow$\\
\midrule
\rowcolor{oursbg} CREDO (full)            & \textbf{.753} & \textbf{.094} & \textbf{.933} & \textbf{.088} & \textbf{.614} & \textbf{.052} & \textbf{.869} & \textbf{.117} \\
Digit-exp.\ emission                       & .676 & .251 & .671 & .218 & .587 & .180 & .724 & .229 \\
Readout + reward only                      & .695 & .262 & .869 & .252 & .588 & .187 & .818 & .186 \\
GRPO + auxiliary head                      & .702 & .106 & .902 & .111 & .599 & .125 & .832 & .155 \\
\bottomrule
\end{tabular}
\end{minipage}

\vspace{1.5em}

\subsection{Discrepancy weighting on verbalized baselines}
\label{app:sw_baseline}

Table~\ref{tab:ablation} shows that discrepancy weighting contributes
substantially to accuracy. A natural question is whether the same weighting
transfers to verbalized baselines. Table~\ref{tab:sw_baseline} adds it to RLCR
and DCPO, varying the source of the confidence value $c$ in the weight
$w \propto |a - c|$:
\begin{itemize}[leftmargin=1.5em,itemsep=1pt,topsep=2pt]
\item \textbf{Parsed numeral.} The weight uses each baseline's own parsed
confidence value. With only a handful of distinct values per group, rollouts
within a group share one of only a few distinct weights.
\item \textbf{Digit-probability expectation.} The weight uses the continuous
expectation over each baseline's own digit logits, so that rollouts within a
group receive distinct weights even when their verbalized reports coincide.
\end{itemize}

Adding discrepancy weighting to RLCR or DCPO does not reproduce the accuracy
improvement that CREDO obtains from the same mechanism
(Table~\ref{tab:ablation}), because the verbalized channel cannot supply the
per-rollout confidence variation the weighting needs. With parsed numerals,
collapsed values leave rollouts in a group sharing one of only a few
distinct weights, so the reweighting barely acts. Digit-expectation confidence gives the weights more
variation, but does not close the gap either: for DCPO, accuracy drops in both
domains, and calibration worsens in mathematics though it improves on
several code metrics. Whether the confidence is read as a parsed numeral or as a
continuous expectation, the verbalized channel does not give the weighting
enough signal to redistribute learning. CREDO's readout varies continuously
across rollouts, which lets the weighting concentrate updates on those where
confidence and outcome disagree most---and that is where the accuracy gains
come from.

\noindent\begin{minipage}{\textwidth}
\centering\scriptsize
\setlength{\tabcolsep}{4pt}
\captionof{table}{Verbalized baselines with discrepancy weighting (seed 43).
The Eval column indicates the confidence source at evaluation.
\colorbox{digitbg}{Shaded rows} evaluate through the digit-probability
expectation; \colorbox{oursbg}{CREDO} through its logit readout; unshaded rows
through the parsed verbalized numeral. \textbf{Bold} marks the best in each
column.}
\label{tab:sw_baseline}
\vspace{0.4em}
\begin{tabular}{@{}llrrrrrrrr@{}}
\toprule
& & \multicolumn{4}{c}{Mathematics} & \multicolumn{4}{c}{Code}\\
\cmidrule(lr){3-6}\cmidrule(lr){7-10}
& Eval & Acc\,$\uparrow$ & ECE\,$\downarrow$ & AUROC\,$\uparrow$ & Brier\,$\downarrow$
& Acc\,$\uparrow$ & ECE\,$\downarrow$ & AUROC\,$\uparrow$ & Brier\,$\downarrow$\\
\midrule
RLCR                        & verbalized & .682 & .167 & .837 & .181 & .539 & .070 & .820 & .144 \\
RLCR + weighting (parsed)   & verbalized & .701 & .142 & .837 & .152 & .545 & .104 & .818 & .155 \\
RLCR + weighting (digit-exp)& verbalized & .697 & .141 & .847 & .148 & .542 & .085 & .810 & .149 \\
\rowcolor{digitbg} RLCR                        & digit-exp  & .682 & .164 & .855 & .180 & .539 & .078 & .822 & .144 \\
\rowcolor{digitbg} RLCR + weighting (parsed)   & digit-exp  & .701 & .151 & .857 & .152 & .545 & .108 & .823 & .153 \\
\rowcolor{digitbg} RLCR + weighting (digit-exp)& digit-exp  & .697 & .152 & .866 & .147 & .542 & .090 & .813 & .147 \\
\midrule
DCPO                        & verbalized & .715 & .222 & .846 & .195 & .579 & .180 & .747 & .179 \\
DCPO + weighting (parsed)   & verbalized & .716 & .223 & .816 & .203 & .581 & .150 & .791 & .154 \\
DCPO + weighting (digit-exp)& verbalized & .690 & .277 & .780 & .262 & .560 & .153 & .801 & .153 \\
\rowcolor{digitbg} DCPO                        & digit-exp  & .715 & .184 & .852 & .180 & .579 & .179 & .842 & .178 \\
\rowcolor{digitbg} DCPO + weighting (parsed)   & digit-exp  & .716 & .187 & .865 & .187 & .581 & .147 & .850 & .151 \\
\rowcolor{digitbg} DCPO + weighting (digit-exp)& digit-exp  & .690 & .221 & .785 & .233 & .560 & .149 & .818 & .147 \\
\midrule
\rowcolor{oursbg} CREDO     & readout    & \textbf{.753} & \textbf{.094} & \textbf{.933} & \textbf{.088} & \textbf{.614} & \textbf{.052} & \textbf{.869} & \textbf{.117} \\
\bottomrule
\end{tabular}
\end{minipage}

\vspace{1.5em}

\subsection{Shape of the answer reward}
\label{app:rewardshape}

We compare the answer-segment reward with two variants of it:
\begin{itemize}[leftmargin=1.5em,itemsep=1pt,topsep=2pt]
\item \textbf{additive} (CREDO, \eqref{eq:rewards3}): $r^{\mathrm{ans}}=\beta F+a$;
\item \textbf{gated}: $r^{\mathrm{ans}}=F\,(\beta+a)$;
\item \textbf{no format term}: $r^{\mathrm{ans}}=a$.
\end{itemize}
The additive form is best or tied on every metric in mathematics, and neither
variant dominates it in code (Table~\ref{tab:rewardshape}).
Under the gated form, a format failure zeroes out the entire reward including
the correctness signal, which slows accuracy learning. Removing the format term
has a milder effect on accuracy but yields mixed results on calibration.

\noindent\begin{minipage}{\textwidth}
\centering\scriptsize
\setlength{\tabcolsep}{4pt}
\captionof{table}{Answer-reward shapes on seed 43; all else follows
Table~\ref{tab:hyper}.}
\label{tab:rewardshape}
\vspace{0.4em}
\begin{tabular}{@{}lrrrrrrrr@{}}
\toprule
& \multicolumn{4}{c}{Mathematics} & \multicolumn{4}{c}{Code}\\
\cmidrule(lr){2-5}\cmidrule(lr){6-9}
$r^{\mathrm{ans}}$ & Acc\,$\uparrow$ & ECE\,$\downarrow$ & AUROC\,$\uparrow$ &
Brier\,$\downarrow$ & Acc\,$\uparrow$ & ECE\,$\downarrow$ & AUROC\,$\uparrow$ &
Brier\,$\downarrow$\\
\midrule
\rowcolor{oursbg} $\beta F + a$
& .753 & .094 & .933 & .088 & .614 & .052 & .869 & .117\\
$F\,(\beta + a)$
& .714 & .105 & .910 & .110 & .589 & .038 & .871 & .113\\
$a$
& .723 & .113 & .924 & .091 & .583 & .042 & .889 & .102\\
\bottomrule
\end{tabular}
\end{minipage}

\vspace{1.5em}

\subsection{Calibration curves}
\label{app:calibcurves}

Figure~\ref{fig:reliability} shows calibration curves for all methods in both
domains. The verbalized baselines resolve only a handful of points, reflecting
the collapse documented in Figure~\ref{fig:collapse}: with few distinct
confidence values, each bin absorbs a large share of the evaluation pool. Most
of these points fall below the diagonal, indicating systematic overconfidence.
CREDO's readout, by contrast, populates the full confidence range and tracks the
diagonal closely. We use equal-mass bins here so that every marker represents a
comparable number of samples; the quantitative ECE in Table~\ref{tab:main} uses
the standard ten equal-width bins defined in Appendix~\ref{app:data}.

\begin{center}
\includegraphics[width=\textwidth]{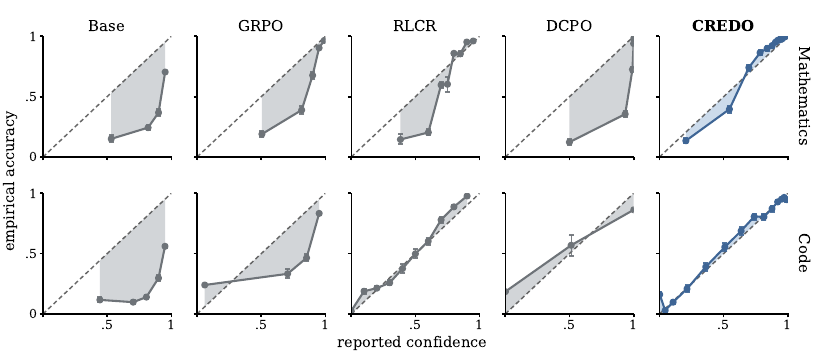}
\vspace{0.2em}
\captionof{figure}{Calibration curves, pooled over the three seeds and over
each domain's evaluation sets. Baselines report through their verbalized
numeral, CREDO through its readout. Equal-mass bins; bars are Wilson 95\%
intervals \citep{wilson1927}.}
\label{fig:reliability}
\end{center}

\section{Extended quantitative results}
\label{app:results}

\subsection{Per-set results}
\label{app:perset}

Tables~\ref{tab:persetmath}--\ref{tab:persetmathout} decompose
Table~\ref{tab:main} by evaluation set, in domain and under transfer. Entries
are mean $\pm$ std over the three seeds. The Average column reproduces
Table~\ref{tab:main}. DeepScaleR and DeepCoder name the in-distribution test
splits.

\noindent\begin{minipage}{\textwidth}
\centering\scriptsize
\setlength{\tabcolsep}{2.6pt}
\captionof{table}{Per-set results on mathematics.}
\label{tab:persetmath}
\vspace{0.4em}
\begin{tabular}{@{}l*{8}{r}@{}}
\toprule
& DeepScaleR & MATH-500 & AIME-24 & AIME-25 & AIME-26 & AMC-23 & AMC-24 & Average\\
\cmidrule(lr){2-9}
\midrule
\multicolumn{9}{@{}l}{\textsc{Accuracy} $\uparrow$}\\
\midrule
Base & .603\sd{.005} & .821\sd{.004} & .240\sd{.013} & .172\sd{.009} & .136\sd{.002} & .599\sd{.010} & .503\sd{.008} & .439\sd{.001}\\
GRPO & .836\sd{.006} & .939\sd{.001} & .610\sd{.021} & .451\sd{.021} & .475\sd{.030} & .856\sd{.014} & .779\sd{.014} & .706\sd{.009}\\
RLCR & .821\sd{.011} & .932\sd{.006} & .564\sd{.021} & .382\sd{.010} & .465\sd{.039} & .851\sd{.021} & .762\sd{.008} & .682\sd{.010}\\
DCPO & .845\sd{.003} & .944\sd{.007} & .624\sd{.035} & .478\sd{.009} & .521\sd{.019} & .864\sd{.018} & .780\sd{.029} & .722\sd{.008}\\
\rowcolor{digitbg} Base-d & .603\sd{.005} & .821\sd{.004} & .240\sd{.013} & .172\sd{.009} & .136\sd{.002} & .599\sd{.010} & .503\sd{.008} & .439\sd{.001}\\
\rowcolor{digitbg} GRPO-d & .836\sd{.006} & .939\sd{.001} & .610\sd{.021} & .451\sd{.021} & .475\sd{.030} & .856\sd{.014} & .779\sd{.014} & .706\sd{.009}\\
\rowcolor{digitbg} RLCR-d & .821\sd{.011} & .932\sd{.006} & .564\sd{.021} & .382\sd{.010} & .465\sd{.039} & .851\sd{.021} & .762\sd{.008} & .682\sd{.010}\\
\rowcolor{digitbg} DCPO-d & .845\sd{.003} & .944\sd{.007} & .624\sd{.035} & .478\sd{.009} & .521\sd{.019} & .864\sd{.018} & .780\sd{.029} & .722\sd{.008}\\
\rowcolor{oursbg} CREDO & \textbf{.850\sd{.003}} & \textbf{.953\sd{.002}} & \textbf{.682\sd{.032}} & \textbf{.504\sd{.015}} & \textbf{.561\sd{.042}} & \textbf{.909\sd{.015}} & \textbf{.801\sd{.008}} & \textbf{.752\sd{.013}}\\
\midrule
\multicolumn{9}{@{}l}{\textsc{ECE} $\downarrow$}\\
\midrule
Base & .317\sd{.005} & .118\sd{.004} & .660\sd{.011} & .714\sd{.018} & .751\sd{.005} & .329\sd{.009} & .416\sd{.013} & .472\sd{.002}\\
GRPO & .093\sd{.017} & \textbf{.028\sd{.007}} & .254\sd{.054} & .396\sd{.066} & .388\sd{.037} & \textbf{.083\sd{.019}} & .149\sd{.028} & .199\sd{.027}\\
RLCR & .054\sd{.003} & .078\sd{.033} & .196\sd{.029} & .345\sd{.030} & .266\sd{.032} & .103\sd{.017} & \textbf{.072\sd{.024}} & .159\sd{.012}\\
DCPO & .129\sd{.002} & .048\sd{.006} & .279\sd{.039} & .388\sd{.014} & .379\sd{.029} & .107\sd{.015} & .189\sd{.019} & .217\sd{.008}\\
\rowcolor{digitbg} Base-d & .276\sd{.005} & .113\sd{.006} & .626\sd{.015} & .694\sd{.010} & .731\sd{.005} & .284\sd{.010} & .374\sd{.011} & .442\sd{.001}\\
\rowcolor{digitbg} GRPO-d & .058\sd{.012} & .076\sd{.011} & .252\sd{.036} & .390\sd{.047} & .368\sd{.021} & \textbf{.088\sd{.053}} & .109\sd{.018} & .192\sd{.025}\\
\rowcolor{digitbg} RLCR-d & .070\sd{.005} & .101\sd{.017} & .201\sd{.002} & .336\sd{.025} & .256\sd{.024} & .107\sd{.014} & \textbf{.073\sd{.011}} & .163\sd{.005}\\
\rowcolor{digitbg} DCPO-d & .065\sd{.012} & .083\sd{.012} & .249\sd{.021} & .359\sd{.037} & .329\sd{.029} & .121\sd{.021} & .108\sd{.015} & .188\sd{.010}\\
\rowcolor{oursbg} CREDO & \textbf{.033\sd{.016}} & .051\sd{.021} & \textbf{.128\sd{.017}} & \textbf{.129\sd{.021}} & \textbf{.171\sd{.007}} & .111\sd{.029} & .099\sd{.026} & \textbf{.103\sd{.013}}\\
\midrule
\multicolumn{9}{@{}l}{\textsc{AUROC} $\uparrow$}\\
\midrule
Base & .619\sd{.018} & .663\sd{.018} & .674\sd{.015} & .653\sd{.025} & .641\sd{.011} & .620\sd{.024} & .640\sd{.017} & .644\sd{.004}\\
GRPO & .699\sd{.044} & .732\sd{.031} & .820\sd{.090} & .826\sd{.072} & .840\sd{.069} & .787\sd{.102} & .722\sd{.031} & .775\sd{.062}\\
RLCR & .737\sd{.021} & .843\sd{.008} & .863\sd{.029} & .871\sd{.003} & .877\sd{.037} & .848\sd{.018} & .781\sd{.017} & .831\sd{.009}\\
DCPO & .752\sd{.009} & .813\sd{.037} & .913\sd{.033} & .882\sd{.017} & .921\sd{.024} & .863\sd{.042} & .793\sd{.049} & .848\sd{.012}\\
\rowcolor{digitbg} Base-d & .703\sd{.007} & .829\sd{.009} & .791\sd{.020} & .779\sd{.020} & .768\sd{.021} & .770\sd{.007} & .711\sd{.039} & .765\sd{.010}\\
\rowcolor{digitbg} GRPO-d & .768\sd{.027} & .866\sd{.013} & .900\sd{.029} & .900\sd{.013} & .906\sd{.024} & .907\sd{.024} & .804\sd{.024} & .864\sd{.020}\\
\rowcolor{digitbg} RLCR-d & .754\sd{.011} & .870\sd{.009} & .877\sd{.020} & .891\sd{.010} & .901\sd{.024} & .877\sd{.019} & .798\sd{.007} & .852\sd{.006}\\
\rowcolor{digitbg} DCPO-d & .749\sd{.009} & .847\sd{.038} & .903\sd{.028} & .903\sd{.007} & .938\sd{.019} & .887\sd{.053} & .776\sd{.043} & .858\sd{.006}\\
\rowcolor{oursbg} CREDO & \textbf{.846\sd{.011}} & \textbf{.926\sd{.017}} & \textbf{.961\sd{.002}} & \textbf{.941\sd{.027}} & \textbf{.968\sd{.007}} & \textbf{.954\sd{.009}} & \textbf{.935\sd{.013}} & \textbf{.933\sd{.007}}\\
\midrule
\multicolumn{9}{@{}l}{\textsc{Brier} $\downarrow$}\\
\midrule
Base & .326\sd{.004} & .150\sd{.002} & .608\sd{.010} & .657\sd{.018} & .691\sd{.006} & .336\sd{.010} & .406\sd{.011} & .453\sd{.002}\\
GRPO & .133\sd{.009} & .051\sd{.002} & .251\sd{.053} & .359\sd{.071} & .348\sd{.045} & .116\sd{.015} & .172\sd{.021} & .204\sd{.029}\\
RLCR & .130\sd{.009} & .057\sd{.005} & .214\sd{.012} & .291\sd{.024} & .250\sd{.018} & .106\sd{.011} & .151\sd{.006} & .171\sd{.009}\\
DCPO & .130\sd{.001} & .047\sd{.006} & .250\sd{.029} & .345\sd{.005} & .330\sd{.048} & .104\sd{.014} & .181\sd{.012} & .198\sd{.003}\\
\rowcolor{digitbg} Base-d & .304\sd{.004} & .143\sd{.003} & .565\sd{.013} & .619\sd{.010} & .648\sd{.006} & .310\sd{.008} & .376\sd{.009} & .424\sd{.001}\\
\rowcolor{digitbg} GRPO-d & .127\sd{.008} & .054\sd{.001} & .262\sd{.041} & .360\sd{.054} & .340\sd{.029} & .111\sd{.017} & .169\sd{.015} & .203\sd{.023}\\
\rowcolor{digitbg} RLCR-d & .129\sd{.008} & .059\sd{.004} & .218\sd{.015} & .288\sd{.022} & .245\sd{.016} & .107\sd{.011} & .150\sd{.004} & .171\sd{.009}\\
\rowcolor{digitbg} DCPO-d & .122\sd{.002} & .050\sd{.004} & .249\sd{.019} & .339\sd{.045} & .308\sd{.048} & .106\sd{.011} & .165\sd{.006} & .191\sd{.013}\\
\rowcolor{oursbg} CREDO & \textbf{.093\sd{.004}} & \textbf{.030\sd{.003}} & \textbf{.091\sd{.007}} & \textbf{.122\sd{.010}} & \textbf{.114\sd{.007}} & \textbf{.060\sd{.012}} & \textbf{.092\sd{.007}} & \textbf{.086\sd{.004}}\\
\bottomrule
\end{tabular}
\end{minipage}

\noindent\begin{minipage}{\textwidth}
\centering\scriptsize
\setlength{\tabcolsep}{3pt}
\captionof{table}{Per-set results on code, models trained on mathematics.}
\label{tab:persetcodeout}
\vspace{0.4em}
\begin{tabular}{@{}l*{10}{r}@{}}
\toprule
& \multicolumn{5}{c}{Accuracy $\uparrow$} & \multicolumn{5}{c}{ECE $\downarrow$}\\
\cmidrule(lr){2-6}\cmidrule(lr){7-11}
& DeepCoder & HumanEval+ & LCB v5 & LCB v6 & Average & DeepCoder & HumanEval+ & LCB v5 & LCB v6 & Average\\
\midrule
GRPO & \textbf{.389\sd{.015}} & .852\sd{.003} & .400\sd{.002} & .338\sd{.007} & .495\sd{.004} & .418\sd{.058} & .086\sd{.011} & .360\sd{.042} & .396\sd{.057} & .315\sd{.041}\\
RLCR & \textbf{.388\sd{.006}} & .858\sd{.003} & \textbf{.400\sd{.009}} & \textbf{.365\sd{.007}} & \textbf{.502\sd{.002}} & .311\sd{.027} & \textbf{.054\sd{.017}} & .271\sd{.017} & .276\sd{.022} & .228\sd{.014}\\
DCPO & \textbf{.394\sd{.022}} & \textbf{.861\sd{.004}} & \textbf{.415\sd{.025}} & \textbf{.363\sd{.019}} & \textbf{.508\sd{.014}} & .436\sd{.059} & .153\sd{.017} & .344\sd{.047} & .387\sd{.049} & .330\sd{.040}\\
\rowcolor{digitbg} GRPO-d & \textbf{.389\sd{.015}} & .852\sd{.003} & .400\sd{.002} & .338\sd{.007} & .495\sd{.004} & .405\sd{.040} & \textbf{.050\sd{.002}} & .345\sd{.031} & .391\sd{.042} & .298\sd{.028}\\
\rowcolor{digitbg} RLCR-d & \textbf{.388\sd{.006}} & .858\sd{.003} & \textbf{.400\sd{.009}} & \textbf{.365\sd{.007}} & \textbf{.502\sd{.002}} & .315\sd{.022} & \textbf{.056\sd{.011}} & .270\sd{.016} & .285\sd{.023} & .232\sd{.013}\\
\rowcolor{digitbg} DCPO-d & \textbf{.394\sd{.022}} & \textbf{.861\sd{.004}} & \textbf{.415\sd{.025}} & \textbf{.363\sd{.019}} & \textbf{.508\sd{.014}} & .384\sd{.036} & \textbf{.050\sd{.021}} & .304\sd{.031} & .361\sd{.039} & .275\sd{.027}\\
\rowcolor{oursbg} CREDO & \textbf{.401\sd{.025}} & \textbf{.853\sd{.016}} & \textbf{.406\sd{.030}} & \textbf{.357\sd{.018}} & \textbf{.504\sd{.021}} & \textbf{.178\sd{.019}} & .062\sd{.010} & \textbf{.148\sd{.016}} & \textbf{.165\sd{.025}} & \textbf{.138\sd{.016}}\\
\midrule
& \multicolumn{5}{c}{AUROC $\uparrow$} & \multicolumn{5}{c}{Brier $\downarrow$}\\
\cmidrule(lr){2-6}\cmidrule(lr){7-11}
& DeepCoder & HumanEval+ & LCB v5 & LCB v6 & Average & DeepCoder & HumanEval+ & LCB v5 & LCB v6 & Average\\
\midrule
GRPO & .765\sd{.033} & .611\sd{.036} & .733\sd{.026} & .777\sd{.041} & .721\sd{.032} & .390\sd{.058} & .127\sd{.005} & .349\sd{.039} & .365\sd{.059} & .308\sd{.040}\\
RLCR & .832\sd{.005} & .717\sd{.025} & .795\sd{.003} & .826\sd{.011} & .792\sd{.005} & .279\sd{.017} & .118\sd{.003} & .263\sd{.011} & .255\sd{.015} & .229\sd{.010}\\
DCPO & .776\sd{.052} & .626\sd{.058} & .784\sd{.025} & .799\sd{.031} & .746\sd{.038} & .415\sd{.047} & .153\sd{.018} & .338\sd{.036} & .366\sd{.044} & .318\sd{.032}\\
\rowcolor{digitbg} GRPO-d & .827\sd{.029} & .702\sd{.035} & .794\sd{.019} & .846\sd{.024} & .792\sd{.026} & .380\sd{.041} & .122\sd{.003} & .338\sd{.028} & .358\sd{.042} & .299\sd{.028}\\
\rowcolor{digitbg} RLCR-d & .850\sd{.006} & \textbf{.754\sd{.026}} & .821\sd{.006} & .838\sd{.017} & .816\sd{.003} & .281\sd{.015} & \textbf{.110\sd{.002}} & .258\sd{.011} & .255\sd{.017} & .226\sd{.011}\\
\rowcolor{digitbg} DCPO-d & .803\sd{.023} & .688\sd{.022} & .807\sd{.026} & .812\sd{.030} & .778\sd{.021} & .366\sd{.025} & .119\sd{.004} & .309\sd{.019} & .337\sd{.027} & .283\sd{.017}\\
\rowcolor{oursbg} CREDO & \textbf{.869\sd{.010}} & \textbf{.735\sd{.033}} & \textbf{.858\sd{.013}} & \textbf{.883\sd{.003}} & \textbf{.836\sd{.009}} & \textbf{.191\sd{.008}} & .120\sd{.011} & \textbf{.176\sd{.009}} & \textbf{.173\sd{.010}} & \textbf{.165\sd{.008}}\\
\bottomrule
\end{tabular}
\end{minipage}

\noindent\begin{minipage}{\textwidth}
\centering\scriptsize
\setlength{\tabcolsep}{3pt}
\captionof{table}{Per-set results on code.}
\label{tab:persetcode}
\vspace{0.4em}
\begin{tabular}{@{}l*{10}{r}@{}}
\toprule
& \multicolumn{5}{c}{Accuracy $\uparrow$} & \multicolumn{5}{c}{ECE $\downarrow$}\\
\cmidrule(lr){2-6}\cmidrule(lr){7-11}
& DeepCoder & HumanEval+ & LCB v5 & LCB v6 & Average & DeepCoder & HumanEval+ & LCB v5 & LCB v6 & Average\\
\midrule
Base & .318\sd{.004} & .829\sd{.007} & .345\sd{.012} & .330\sd{.007} & .455\sd{.001} & .558\sd{.005} & .105\sd{.008} & .503\sd{.012} & .506\sd{.001} & .418\sd{.001}\\
GRPO & .475\sd{.006} & \textbf{.882\sd{.011}} & .522\sd{.008} & \textbf{.443\sd{.007}} & .581\sd{.007} & .229\sd{.070} & .113\sd{.098} & .187\sd{.083} & .199\sd{.036} & .182\sd{.070}\\
RLCR & .449\sd{.008} & \textbf{.872\sd{.019}} & .480\sd{.017} & .412\sd{.016} & .553\sd{.013} & \textbf{.049\sd{.014}} & .111\sd{.017} & \textbf{.068\sd{.010}} & \textbf{.071\sd{.016}} & \textbf{.075\sd{.005}}\\
DCPO & .477\sd{.016} & \textbf{.878\sd{.016}} & .525\sd{.013} & .422\sd{.010} & .576\sd{.003} & .167\sd{.031} & .112\sd{.015} & .178\sd{.041} & .156\sd{.013} & .153\sd{.024}\\
\rowcolor{digitbg} Base-d & .318\sd{.004} & .829\sd{.007} & .345\sd{.012} & .330\sd{.007} & .455\sd{.001} & .528\sd{.004} & .074\sd{.017} & .474\sd{.012} & .481\sd{.003} & .389\sd{.002}\\
\rowcolor{digitbg} GRPO-d & .475\sd{.006} & \textbf{.882\sd{.011}} & .522\sd{.008} & \textbf{.443\sd{.007}} & .581\sd{.007} & .218\sd{.038} & \textbf{.040\sd{.019}} & .150\sd{.022} & .213\sd{.030} & .155\sd{.022}\\
\rowcolor{digitbg} RLCR-d & .449\sd{.008} & \textbf{.872\sd{.019}} & .480\sd{.017} & .412\sd{.016} & .553\sd{.013} & \textbf{.056\sd{.013}} & .115\sd{.014} & \textbf{.066\sd{.012}} & \textbf{.082\sd{.017}} & .080\sd{.004}\\
\rowcolor{digitbg} DCPO-d & .477\sd{.016} & \textbf{.878\sd{.016}} & .525\sd{.013} & .422\sd{.010} & .576\sd{.003} & .164\sd{.031} & .112\sd{.015} & .176\sd{.041} & .157\sd{.013} & .152\sd{.025}\\
\rowcolor{oursbg} CREDO & \textbf{.509\sd{.027}} & \textbf{.881\sd{.006}} & \textbf{.563\sd{.016}} & \textbf{.448\sd{.024}} & \textbf{.600\sd{.018}} & \textbf{.049\sd{.012}} & \textbf{.062\sd{.046}} & \textbf{.069\sd{.009}} & \textbf{.076\sd{.017}} & \textbf{.064\sd{.020}}\\
\midrule
& \multicolumn{5}{c}{AUROC $\uparrow$} & \multicolumn{5}{c}{Brier $\downarrow$}\\
\cmidrule(lr){2-6}\cmidrule(lr){7-11}
& DeepCoder & HumanEval+ & LCB v5 & LCB v6 & Average & DeepCoder & HumanEval+ & LCB v5 & LCB v6 & Average\\
\midrule
Base & .701\sd{.010} & .628\sd{.019} & .695\sd{.006} & .735\sd{.012} & .690\sd{.004} & .513\sd{.006} & .146\sd{.006} & .464\sd{.009} & .454\sd{.007} & .395\sd{.002}\\
GRPO & .791\sd{.094} & .605\sd{.033} & .794\sd{.083} & .822\sd{.057} & .753\sd{.061} & .237\sd{.067} & .156\sd{.097} & .215\sd{.074} & .211\sd{.035} & .205\sd{.066}\\
RLCR & .845\sd{.003} & \textbf{.763\sd{.028}} & .825\sd{.017} & .825\sd{.020} & .815\sd{.014} & .156\sd{.003} & .110\sd{.004} & .166\sd{.008} & .150\sd{.012} & .146\sd{.006}\\
DCPO & .834\sd{.032} & .603\sd{.040} & .829\sd{.040} & .839\sd{.017} & .776\sd{.025} & .168\sd{.030} & .113\sd{.014} & .177\sd{.041} & .155\sd{.014} & .153\sd{.024}\\
\rowcolor{digitbg} Base-d & .796\sd{.010} & .719\sd{.022} & .782\sd{.006} & .832\sd{.006} & .782\sd{.003} & .482\sd{.004} & .141\sd{.006} & .437\sd{.009} & .432\sd{.006} & .373\sd{.002}\\
\rowcolor{digitbg} GRPO-d & .874\sd{.033} & .713\sd{.047} & .861\sd{.031} & .884\sd{.027} & .833\sd{.034} & .230\sd{.036} & \textbf{.098\sd{.010}} & .192\sd{.024} & .214\sd{.030} & .184\sd{.025}\\
\rowcolor{digitbg} RLCR-d & .848\sd{.001} & \textbf{.775\sd{.027}} & .826\sd{.017} & .821\sd{.018} & .818\sd{.011} & .155\sd{.001} & .110\sd{.003} & .165\sd{.008} & .151\sd{.012} & .145\sd{.005}\\
\rowcolor{digitbg} DCPO-d & \textbf{.883\sd{.034}} & .730\sd{.081} & \textbf{.887\sd{.051}} & \textbf{.896\sd{.031}} & \textbf{.849\sd{.027}} & .165\sd{.031} & .112\sd{.015} & .175\sd{.041} & .155\sd{.013} & .152\sd{.025}\\
\rowcolor{oursbg} CREDO & \textbf{.889\sd{.025}} & .736\sd{.045} & \textbf{.888\sd{.023}} & \textbf{.887\sd{.029}} & \textbf{.850\sd{.027}} & \textbf{.131\sd{.012}} & \textbf{.115\sd{.037}} & \textbf{.132\sd{.010}} & \textbf{.129\sd{.010}} & \textbf{.127\sd{.016}}\\
\bottomrule
\end{tabular}
\end{minipage}

\noindent\begin{minipage}{\textwidth}
\centering\scriptsize
\setlength{\tabcolsep}{2.6pt}
\captionof{table}{Per-set results on mathematics, models trained on code.}
\label{tab:persetmathout}
\vspace{0.4em}
\begin{tabular}{@{}l*{8}{r}@{}}
\toprule
& DeepScaleR & MATH-500 & AIME-24 & AIME-25 & AIME-26 & AMC-23 & AMC-24 & Average\\
\cmidrule(lr){2-9}
\midrule
\multicolumn{9}{@{}l}{\textsc{Accuracy} $\uparrow$}\\
\midrule
GRPO & \textbf{.743\sd{.004}} & \textbf{.892\sd{.005}} & .374\sd{.016} & \textbf{.276\sd{.009}} & \textbf{.288\sd{.004}} & \textbf{.728\sd{.003}} & \textbf{.644\sd{.010}} & \textbf{.564\sd{.003}}\\
RLCR & .701\sd{.016} & .873\sd{.009} & .324\sd{.009} & .233\sd{.026} & .207\sd{.016} & .688\sd{.024} & .579\sd{.025} & .515\sd{.011}\\
DCPO & .721\sd{.009} & \textbf{.894\sd{.002}} & .357\sd{.010} & .240\sd{.024} & .253\sd{.013} & \textbf{.727\sd{.004}} & \textbf{.644\sd{.017}} & .548\sd{.003}\\
\rowcolor{digitbg} GRPO-d & \textbf{.743\sd{.004}} & \textbf{.892\sd{.005}} & .374\sd{.016} & \textbf{.276\sd{.009}} & \textbf{.288\sd{.004}} & \textbf{.728\sd{.003}} & \textbf{.644\sd{.010}} & \textbf{.564\sd{.003}}\\
\rowcolor{digitbg} RLCR-d & .701\sd{.016} & .873\sd{.009} & .324\sd{.009} & .233\sd{.026} & .207\sd{.016} & .688\sd{.024} & .579\sd{.025} & .515\sd{.011}\\
\rowcolor{digitbg} DCPO-d & .721\sd{.009} & \textbf{.894\sd{.002}} & .357\sd{.010} & .240\sd{.024} & .253\sd{.013} & \textbf{.727\sd{.004}} & \textbf{.644\sd{.017}} & .548\sd{.003}\\
\rowcolor{oursbg} CREDO & .717\sd{.021} & \textbf{.888\sd{.013}} & \textbf{.401\sd{.038}} & \textbf{.268\sd{.017}} & .269\sd{.023} & \textbf{.737\sd{.020}} & \textbf{.642\sd{.027}} & \textbf{.560\sd{.007}}\\
\midrule
\multicolumn{9}{@{}l}{\textsc{ECE} $\downarrow$}\\
\midrule
GRPO & .178\sd{.010} & \textbf{.048\sd{.009}} & .496\sd{.035} & .587\sd{.038} & .571\sd{.033} & .200\sd{.012} & .263\sd{.007} & .335\sd{.019}\\
RLCR & \textbf{.034\sd{.002}} & .102\sd{.010} & .249\sd{.014} & .306\sd{.020} & .328\sd{.019} & \textbf{.055\sd{.018}} & \textbf{.117\sd{.003}} & .170\sd{.004}\\
DCPO & .224\sd{.006} & .092\sd{.009} & .237\sd{.044} & \textbf{.197\sd{.061}} & .237\sd{.034} & .216\sd{.015} & .228\sd{.023} & .205\sd{.021}\\
\rowcolor{digitbg} GRPO-d & .133\sd{.011} & .063\sd{.019} & .468\sd{.032} & .569\sd{.023} & .551\sd{.022} & .165\sd{.014} & .227\sd{.008} & .311\sd{.017}\\
\rowcolor{digitbg} RLCR-d & .049\sd{.005} & .115\sd{.007} & .250\sd{.008} & .308\sd{.015} & .331\sd{.018} & \textbf{.056\sd{.014}} & \textbf{.098\sd{.020}} & .172\sd{.006}\\
\rowcolor{digitbg} DCPO-d & .224\sd{.006} & .091\sd{.008} & .234\sd{.039} & \textbf{.196\sd{.059}} & .238\sd{.031} & .214\sd{.016} & .225\sd{.023} & .203\sd{.020}\\
\rowcolor{oursbg} CREDO & .127\sd{.071} & .109\sd{.086} & \textbf{.100\sd{.036}} & \textbf{.146\sd{.071}} & \textbf{.126\sd{.067}} & .118\sd{.073} & \textbf{.093\sd{.063}} & \textbf{.117\sd{.030}}\\
\midrule
\multicolumn{9}{@{}l}{\textsc{AUROC} $\uparrow$}\\
\midrule
GRPO & .682\sd{.037} & .739\sd{.048} & .773\sd{.014} & .780\sd{.039} & .767\sd{.035} & .682\sd{.046} & .712\sd{.028} & .734\sd{.030}\\
RLCR & .734\sd{.013} & .857\sd{.005} & .837\sd{.029} & .889\sd{.041} & .843\sd{.017} & .747\sd{.020} & .797\sd{.032} & .815\sd{.016}\\
DCPO & .686\sd{.025} & .747\sd{.027} & .769\sd{.037} & .837\sd{.050} & .815\sd{.008} & .711\sd{.022} & .720\sd{.035} & .755\sd{.024}\\
\rowcolor{digitbg} GRPO-d & \textbf{.755\sd{.003}} & .857\sd{.015} & .848\sd{.024} & .887\sd{.009} & .841\sd{.004} & .819\sd{.014} & .776\sd{.015} & .826\sd{.004}\\
\rowcolor{digitbg} RLCR-d & .748\sd{.010} & \textbf{.876\sd{.007}} & .856\sd{.018} & .907\sd{.036} & \textbf{.861\sd{.025}} & .762\sd{.018} & \textbf{.819\sd{.027}} & .833\sd{.014}\\
\rowcolor{digitbg} DCPO-d & .750\sd{.005} & .859\sd{.028} & .818\sd{.019} & \textbf{.916\sd{.047}} & \textbf{.866\sd{.016}} & \textbf{.830\sd{.022}} & \textbf{.803\sd{.036}} & .835\sd{.004}\\
\rowcolor{oursbg} CREDO & .736\sd{.022} & \textbf{.874\sd{.011}} & \textbf{.900\sd{.009}} & \textbf{.944\sd{.013}} & \textbf{.850\sd{.026}} & \textbf{.845\sd{.028}} & \textbf{.834\sd{.049}} & \textbf{.855\sd{.013}}\\
\midrule
\multicolumn{9}{@{}l}{\textsc{Brier} $\downarrow$}\\
\midrule
GRPO & .207\sd{.006} & \textbf{.090\sd{.006}} & .452\sd{.036} & .525\sd{.042} & .510\sd{.039} & .220\sd{.011} & .268\sd{.008} & .325\sd{.020}\\
RLCR & \textbf{.176\sd{.008}} & \textbf{.090\sd{.002}} & .217\sd{.011} & .216\sd{.011} & .234\sd{.012} & .177\sd{.008} & .191\sd{.008} & .186\sd{.007}\\
DCPO & .228\sd{.005} & \textbf{.091\sd{.008}} & .235\sd{.043} & .197\sd{.060} & .237\sd{.031} & .216\sd{.015} & .229\sd{.024} & .205\sd{.020}\\
\rowcolor{digitbg} GRPO-d & .193\sd{.007} & \textbf{.087\sd{.007}} & .426\sd{.031} & .503\sd{.028} & .485\sd{.030} & .208\sd{.005} & .257\sd{.006} & .308\sd{.015}\\
\rowcolor{digitbg} RLCR-d & \textbf{.175\sd{.007}} & \textbf{.089\sd{.002}} & .213\sd{.012} & .214\sd{.009} & .232\sd{.013} & .175\sd{.009} & .186\sd{.006} & .184\sd{.006}\\
\rowcolor{digitbg} DCPO-d & .227\sd{.006} & \textbf{.091\sd{.008}} & .232\sd{.039} & .195\sd{.059} & .235\sd{.030} & .213\sd{.015} & .226\sd{.023} & .203\sd{.020}\\
\rowcolor{oursbg} CREDO & \textbf{.188\sd{.034}} & \textbf{.092\sd{.042}} & \textbf{.130\sd{.007}} & \textbf{.111\sd{.025}} & \textbf{.153\sd{.019}} & \textbf{.155\sd{.032}} & \textbf{.164\sd{.029}} & \textbf{.142\sd{.013}}\\
\bottomrule
\end{tabular}
\end{minipage}

\subsection{Sensitivity to the coefficients}
\label{app:sens}

Table~\ref{tab:sens} gives every evaluated setting of the three coefficients.

\noindent\begin{minipage}{\textwidth}
\centering\scriptsize
\setlength{\tabcolsep}{3.5pt}
\captionof{table}{Every evaluated setting of the three coefficients
of~\S\ref{sec:credo}, all on seed 43.}
\label{tab:sens}
\vspace{0.4em}
\begin{tabular}{@{}lrrrrrrrr@{}}
\toprule
& \multicolumn{4}{c}{Mathematics} & \multicolumn{4}{c}{Code}\\
\cmidrule(lr){2-5}\cmidrule(lr){6-9}
& \hd & \hd\\
\midrule
\multicolumn{9}{@{}l}{\textsc{Weighting strength $\kappa$}}\\
\midrule
$0$ & .721 & .122 & .904 & .102 & .572 & .051 & .843 & .132 \\
$0.25$ & .753 & .094 & .933 & .088 & .584 & .044 & .854 & .124 \\
$0.5$ & .749 & .093 & .924 & .101 & .614 & .052 & .869 & .117 \\
$1$ & .730 & .108 & .924 & .100 & .591 & .049 & .858 & .124 \\
\midrule
\multicolumn{9}{@{}l}{\textsc{Calibration weight $\alpha$}}\\
\midrule
$0$ & .696 & .241 & .880 & .232 & .611 & .172 & .857 & .172 \\
$0.01$ & .715 & .123 & .912 & .115 & .581 & .044 & .851 & .123 \\
$0.03$ & .753 & .094 & .933 & .088 & .614 & .052 & .869 & .117 \\
$0.05$ & .708 & .103 & .910 & .107 & .583 & .042 & .857 & .123 \\
$0.10$ & .716 & .089 & .915 & .098 & .577 & .042 & .867 & .117 \\
\midrule
\multicolumn{9}{@{}l}{\textsc{Target mixture $\gamma$}}\\
\midrule
$0$ & .732 & .130 & .918 & .098 & .593 & .054 & .854 & .125 \\
$0.5$ & .753 & .094 & .933 & .088 & .614 & .052 & .869 & .117 \\
$1$ & .738 & .101 & .924 & .100 & .587 & .034 & .855 & .120 \\
\bottomrule
\end{tabular}
\end{minipage}

\end{document}